\documentclass{article}[12pt]
\usepackage[sf]{titlesec}

\usepackage{amssymb}
\usepackage{latexsym}
\usepackage{amsmath, amsfonts, amsthm}
\usepackage{xspace}
\usepackage{multicol}
\usepackage{colordvi}
\usepackage{epsf}
\usepackage{graphicx}
\usepackage{color,psfrag} % for laprint figures
\usepackage{xcolor}
\usepackage{graphics}
\usepackage{amsthm}
\usepackage{amscd}
\usepackage{sidecap}
\usepackage{wrapfig}
\usepackage{subfig}
\usepackage{enumitem}

\newtheoremstyle{slplain}% name
{.1\baselineskip\@plus.2\baselineskip\@minus.2\baselineskip}% Space above
{.1\baselineskip\@plus.2\baselineskip\@minus.2\baselineskip}% Space below
{}% Body font
{}%Indent amount (empty = no indent, \parindent = para indent)
{\bfseries}%  Thm head font
{.}%       Punctuation after thm head
{ }%      Space after thm head: " " = normal interword space;
{}%       Thm head spec

\theoremstyle{slplain}

\titlespacing{\section}{0pt}{*0.45}{*0.45}
\titlespacing{\subsection}{0pt}{*0.45}{*0.45}
\titlespacing{\subsubsection}{0pt}{*0.45}{*0.45}

\newcommand{\la}{\left \langle}
\newcommand{\ra}{\right\rangle}

\newcommand{\norm}[1]{\left\lVert #1 \right\rVert}

\newtheorem{theorem}{Theorem}[section]

\newtheorem{lemma}[theorem]{Lemma}

\theoremstyle{definition}

\newtheorem{assumption}[theorem]{Assumption}

\theoremstyle{remark}
\newtheorem{remark}[theorem]{Remark}

\numberwithin{equation}{section}

\usepackage{mathtools}

\DeclarePairedDelimiter\floor{\lfloor}{\rfloor}

\title{Asymptotics of Reinforcement Learning with Neural Networks}

\author{Justin Sirignano\footnote{Mathematics, University of Oxford, E-mail: Justin.Sirignano@maths.ox.ac.uk} \footnote{Department of Industrial \& Systems Engineering, University of Illinois at Urbana-Champaign} \phantom{.}  and Konstantinos Spiliopoulos\footnote{Department of Mathematics and Statistics, Boston University, Boston, E-mail: kspiliop@math.bu.edu}
\thanks{K.S. was partially supported by the National Science Foundation (DMS 1550918) and Simons Foundation Award 672441}\\
}

\begin{document}

\maketitle

%\vspace{-50pt}

%\vspace{2mm}

\begin{abstract}
We prove that a single-layer neural network trained with the Q-learning algorithm converges in distribution to a random ordinary differential equation as the size of the model and the number of training steps become large. Analysis of the limit differential equation shows that it has a unique stationary solution which is the solution of the Bellman equation, thus giving the optimal control for the problem. In addition, we study the convergence of the limit differential equation to the stationary solution. As a by-product of our analysis, we obtain the limiting behavior of single-layer neural networks when trained on i.i.d. data with stochastic gradient descent under the widely-used Xavier initialization.
\end{abstract}

\section{Introduction}

Reinforcement learning with neural networks (frequently called ``deep reinforcement learning") has had a number of recent successes, including learning to play video games \cite{MnihEtAll_videoGames_2015,Atari}, mastering the game of Go \cite{GoGame_2017}, and robotics \cite{KobersPeters2012}. In deep reinforcement learning, a neural  network is trained to learn the optimal action given the current state.

Despite many advances in applications, a number of mathematical questions remain open regarding reinforcement learning with neural networks. Our paper studies the Q-learning algorithm with neural networks (typically called ``deep Q-learning"), which is a popular reinforcement learning method for training a neural network to learn the optimal control for a stochastic optimal control problem. The deep Q-learning algorithm uses a neural network to approximate the value of an action $a$ in a state $x$ \cite{MnihEtAll_videoGames_2015}. This neural network approximator is called the ``$Q$-network". The Q-learning algorithm estimates the $Q$-network by taking stochastic steps which attempt to train the $Q$-network to satisfy the Bellman equation.

The literature on (deep) reinforcement learning and Q-learning is substantial. Instead of providing a complete literature review here we refer interested readers to classical texts \cite{BertsekasTsitsiklis,KushnerYin,SuttonBarto1998}, to the more recent book \cite{Goodfellow}, and to the extensive survey on recent developments in \cite{Arulkumaran2017}.  The majority of reinforcement learning algorithms are based on some variation of the Q-learning or policy gradient methods \cite{Sutton2000}. Q-learning originated in \cite{Watkins} and proofs  of convergence can be found in \cite{WatkinsDayan,Tsitsiklis1994}. The neural network approach to reinforcement learning (i.e., using Q-networks) was proposed in \cite{MnihEtAll_videoGames_2015}. More recent developments include deep recurrent Q-networks \cite{Hausknecht2015}, dueling architectures for deep reinforcement learning \cite{Wang2016}, double Q-learning \cite{DoubleQLearning}, bootstrapped deep Q-networks \cite{Osband}, and asynchronous methods for deep reinforcement learning \cite{Mnih2016}. Although the performance of $Q-$networks has been extensively studied in numerical experiments, there has been relatively little theoretical investigation.

We study the behavior of a single-layer $Q-$network in the asymptotic regime of large numbers of hidden units and large numbers of training steps. We prove that the $Q-$network (which models the value function for the related optimal control problem) converges to the solution of a random ordinary differential equation (ODE). We characterize the limiting random ODE in both the infinite and finite time horizon discounted reward cases. Then, we study the behavior of the solution to the limiting random ODE as time $t\rightarrow\infty$.

In the infinite time horizon case, we show that the limit ODE has a unique stationary solution which equals the solution of the associated Bellman equation. Thus, the unique stationary solution of the limit $Q$-network gives the optimal control for the problem. In the infinite time horizon case, we also show that the limit ODE converges to the unique stationary solution for small values of the discount factor.

The presence of a neural network in the Q-learning algorithm introduces technical challenges, which lead us to be able to prove, in the infinite time horizon case, convergence of the limiting ODE to the stationary solution only for small values of the discount factor. We elaborate more on this issue in Remark \ref{R:ConvergenceInfiniteHorizon}. The situation is somewhat different in the finite time horizon case, where we can prove that the limit ODE converges to a global minimum, which is the solution of the associated Bellman equation, for all values of the discount factor in $(0,1]$.
%; see Section \ref{S:MainResults}.

As a by-product of our analysis, we also prove that a single-layer neural network trained on i.i.d. data with stochastic gradient descent under the Xavier initialization \cite{Xavier} converges to a limit ODE. In addition to characterizing the limiting behavior of the neural network as the number of hidden units and stochastic gradient descent steps grow to infinity, we also obtain that the neural network in the limit converges to a global minimum with zero training loss (see Section \ref{S:SpecialCase}).

\textcolor{black}{The focus of this work is on the explicit characterization and on the study of the limit of the $Q-$network as the number of hidden units and stochastic gradient descent iterates grow to infinity. In addition, our results show that convergence to a global minimum can also be viewed as a simple consequence of the limit ODE for neural networks. Lastly, we mention that our weak convergence method of proof is general enough that it can cover under the same umbrella the infinite time horizon discounted reward problem, the finite time horizon discounted reward problem (Section \ref{S:MainResults}), and the classical regression problem (Section \ref{S:SpecialCase}).}

The rest of the paper is organized as follows. The Q-learning algorithm is introduced in Section \ref{S:BackgroundInformation}. Section \ref{S:MainResults} presents our main theorems. Section \ref{S:SpecialCase} discusses the limiting behavior of single-layer neural networks when trained on i.i.d. data with stochastic gradient descent under the Xavier initialization. Section \ref{ProofInfiniteCase} contains the proof for the infinite time horizon reinforcement learning case. The proof for the finite time horizon reinforcement learning case is in Section \ref{S:ProofFiniteCase}. Section \ref{ProofOFPositiveDefinite} contains a proof that a certain matrix in the limit ODE is positive definite, which is useful for establishing convergence properties of the limiting ODEs. Section \ref{Conclusion} summarizes the conclusions of this work.

\section{Q-learning Algorithm}\label{S:BackgroundInformation}

We consider a Markov decision problem defined on the finite state space $\mathcal{X} \subset \mathbb{R}^{d_x}$. For every state $x\in\mathcal{X}$ there is a finite set $\mathcal{A} \subset \mathbb{R}^{d_a}$ of actions that can be taken. The homogeneous Markov chain $x_k \in \mathcal{X}$ has a probability transition function $\mathbb{P}[ x_{j+1} = z | x_j =x, a_j =a ] = p(z | x, a)$ which governs the probability that $x_{j+1} = z$ given that $x_j = x$ and  $a_j =a$. For every state $x$ and action $a$, a reward of $r(x,a)$ is collected. Let $\lambda$ denote an admissible control policy (i.e., it is chosen based on a probability law such that it depends only on the history up to the present).

\subsection{The infinite time horizon setting}\label{SS:InfTimeHor}
For a given initial state $x\in\mathcal{X}$ and admissible control policy $\lambda$, the infinite time horizon reward is defined to be
\[
W_{\lambda}(x)=\mathbb{E}_{\lambda} \bigg{[} \sum_{j = 0}^{\infty} \gamma^{j} r(x_j,a_j) | x_0 = x \bigg{]},
\]
where the actions $a_j$ for $j \geq 0$ are chosen according to the policy $\lambda$ and $\gamma\in(0,1]$ is the discount factor.

Let $V(x,a)$ be the reward given that we start at state $x\in\mathcal{X}$, action $a \in \mathcal{A}$ is taken, and the optimal policy is subsequently used. As is well known (see for example Chapter 2.4 of \cite{KushnerYin}), $\displaystyle \max_{a\in\mathcal{A}}V(x,a)=\sup_{\lambda} W_{\lambda}(x)$ and the maximum expected reward $V$ satisfies the Bellman equation

\begin{align}
0 &= r(x,a) + \gamma \sum_{z \in \mathcal{X}} \max_{a' \in \mathcal{A}} V(z,a')  p(z | x, a)  - V(x,a ), %\notag \\
%a^{\ast}(x) &= \arg \max_{a \in \mathcal{A}} V(x,a).
\label{HJBmain}
\end{align}
and where $\displaystyle a^{\ast}(x) = \arg \max_{a \in \mathcal{A}} V(x,a)$ is an optimal policy. The Bellman equation (\ref{HJBmain}) can be derived using the principle of optimality for dynamic programming.

\subsection{The finite time horizon setting}\label{SS:fTimeHor}
In the finite time horizon case, for a given initial state $x\in\mathcal{X}$ and admissible control policy $\lambda$, the finite time horizon reward is defined to be
 %with times $j = 0, 1, \ldots, J$, the objective is to find an optimal control $a^{\ast}(j, x)$  which maximizes the expected discounted cumulative reward
\begin{align}
W_{\lambda}(J,x)=\mathbb{E}_{\lambda} \bigg{[} \sum_{j = 0}^{J} \gamma^{j} r_j | x_0 = x \bigg{]},
\label{DiscountedRewardFiniteT}
\end{align}
where $r_j = r(j,x_j,a_j)$ for $j=0, 1, \ldots, J-1$, $r_{J}=r(J,x_{J})$, and $J$ is the deterministic time horizon. \textcolor{black}{ Later, in Theorem \ref{MainTheorem1}, we will strengthen the requirement for $x\in\mathcal{X}$ to be sampled from a distribution $\pi_{0}$ with $\pi_{0}(x)>0$.}

Similar to the infinite time horizon discount case and for a given $x\in\mathcal{X}$, the optimal control $a^{\ast}(j, x)$ is given by the solution to the Bellman equation
\begin{align}
 V(j,x,a)&= r(j, x,a) + \gamma \sum_{z} \max_{a'} V(j+1, z, a') p(z| x,a), \phantom{.....} j = 0, 1, \ldots, J-1, \notag \\
V(J,x,a) &= r(J,x), %\notag \\
%a^{\ast}(j, x) =& \arg \max_{a \in \mathcal{A}} V(j, x,a).
\label{HJBmainFiniteT}
\end{align}
with the optimal control  given by $\displaystyle a^{\ast}(j, x) =\arg \max_{a \in \mathcal{A}} V(j, x,a)$. By the principle of optimality we also have $\displaystyle \max_{a \in \mathcal{A}} V(0, x,a)=\sup_{\lambda} W_{\lambda}(J,x)$.

In principle, the Bellman equations  (\ref{HJBmain}) and (\ref{HJBmainFiniteT}) can be solved to find the optimal control. However, there are two obstacles. First, the transition probability function $p(z | x,a)$ (i.e., the state dynamics) may not be known. Secondly, even if it is known, the state space may be too high-dimensional for standard numerical methods to solve (\ref{HJBmain}) and (\ref{HJBmainFiniteT}) due to the curse of dimensionality. For these reasons, reinforcement learning methods can be used to \emph{learn} the solution to the Bellman equations (\ref{HJBmain}) and (\ref{HJBmainFiniteT}).

\subsection{Reinforcement learning with neural network approximation}

Reinforcement learning can approximate the solution to the Bellman equation using a function approximator, which typically is a neural network model. The parameters $\theta$ (i.e., the weights) of the neural network are estimated using the Q-learning algorithm. The neural network $Q(x,a;\theta)$ in Q-learning is referred to as a ``Q-network".

The Q-learning algorithm attempts to minimize the objective function
\begin{align}
L(\theta)&= \sum_{(x,a) \in \mathcal{X} \times \mathcal{A}} \left[\left(Y(x,a) -Q(x,a;\theta)\right)^{2}\right] \pi(x,a),\label{Eq:LossFunction}
\end{align}
where $\pi(x,a)$ is a probability mass function (to be specified later on) which is strictly positive for every $(x,a) \in \mathcal{X} \times \mathcal{A}$ and the ``target" $Y$ is
\begin{align*}
Y(x,a) &= r(x,a) + \gamma \sum_{x' \in \mathcal{X}} \max_{a'\in\mathcal{A}} Q(x',a';\theta) p(x' | x,a) . %\label{Eq:TargetData}
\end{align*}

In the case of the infinite time horizon problem, Section \ref{SS:InfTimeHor} (and analogously for the finite time horizon problem, Section \ref{SS:fTimeHor}), if $L(\theta) = 0$, then $Q(x,a; \theta)$ is a solution to the Bellman equation (\ref{HJBmain}). In practice, the hope is that the Q-learning algorithm will learn a model $Q$ such that $L(\theta)$ is small and therefore $Q(x, a; \theta)$ is a good approximation for the Bellman solution $V(x,a)$.

The Q-learning updates for the parameters $\theta$ are:
\begin{eqnarray}
\theta_{k+1} &=&  \theta_k + \alpha_{k}^{N} G_k, \notag \\
G_k &=& \bigg{(} r(x_k, a_k) + \gamma \max_{a'\in\mathcal{A}} Q(x_{k+1} ,a'; \theta_{k}) - Q(x_k, a_k; \theta_k) \bigg{)}  \nabla_{\theta} Q(x_k, a_k; \theta_k),
\label{GeneralQlearningAlgorithm}
\end{eqnarray}
where for $k=0,1,\cdots,$ $(x_k, a_k)$ is an ergodic Markov chain with $\pi(x,a)$ as its limiting distribution and $\alpha_{k}^{N}$ is the learning rate that may depend on the number of hidden units $N$ and on the iteration $k$. In this paper, we take $\alpha_{k}^{N}=\alpha^{N}$ to only depend on $N$.

The $Q-$network, which models the value of a state $x$ and action $a$, is the neural network
\begin{align}
Q^{N}(x,a; \theta) = \frac{1}{\sqrt{N}} \sum_{i=1}^N C^{i} \sigma \big{(} W^i \cdot (x,a) \big{)},
\label{NetworkXavier}
\end{align}
where $C^i \in \mathbb{R}$, $W^i \in \mathbb{R}^d$, $x \in \mathcal{X} \subset \mathbb{R}^{d_x}$, $a \in \mathcal{A} \subset \mathbb{R}^{d_a}$, $d = d_x + d_a$, and $\sigma(\cdot): \mathbb{R} \rightarrow \mathbb{R}$. For generic vectors $a,b\in\mathbb{R}^{d}$, we will be denoting by $a\cdot b$ their usual inner product. \textcolor{black}{$(x, a) \in \mathbb{R}^d$ is a vector which is a concatenation of the vectors $x$ and $a$.} The parametric model (\ref{NetworkXavier}) receives an input vector containing both the state and action in the enlarged Euclidean space $\mathbb{R}^{d}$. This formulation is a common choice in practice; see for example \cite{Dulac2015}. Other variations of the parametric model (\ref{NetworkXavier}), such as an input vector of the state and an output vector which is the length of the number of possible actions, are of course possible and can also be studied using this paper's techniques.

The number of hidden units is $N$ and the output is scaled by a factor $\frac{1}{\sqrt{N}}$, which is commonly used in practice and is called the ``Xavier initialization" \cite{Xavier}. Essentially, the intuition behind the  ``Xavier initialization" \cite{Xavier} for deep neural networks is to ensure that all layers have the same input variance and the same gradient variance; see chapter 8.4 of \cite{Goodfellow} for an extended discussion on the topic of initialization.

 The set of parameters that must be estimated is
\[
\theta=(C^{1},\cdots, C^{N}, W^{1},\cdots,W^{N})\in\mathbb{R}^{(1+d)N}.
\]

In the infinite-time horizon case, the Q-learning algorithm for training the parameters $\theta$ is
\begin{align}
C^{i}_{k+1} =& C^{i}_k + \frac{ \alpha^N }{\sqrt{N}} \bigg{(} r_k + \gamma \max_{a' \in \mathcal{A}} Q^N(x_{k+1}, a'; \theta_k) - Q^N(x_k, a_k; \theta_k) \bigg{)} \sigma \big{(} W^i_k \cdot (x_k, a_k) \big{)}, \notag \\
W^i_{k+1} =& W^i_k + \frac{ \alpha^N}{\sqrt{N}}  \bigg{(} r_k + \gamma \max_{a' \in \mathcal{A}} Q^N(x_{k+1}, a'; \theta_k) - Q^N(x_k, a_k; \theta_k) \bigg{)} C^{i}_k \sigma' \big{(} W^i_k \cdot (x_k, a_k) \big{)}  (x_k, a_k), \notag \\
Q^{N}(x,a; \theta_k) =& \frac{1}{\sqrt{N}} \sum_{i=1}^N C^{i}_k \sigma \big{(} W^i_k  \cdot (x,a) \big{)}, \label{SGDupdates}
\end{align}
for $k = 0, 1, \ldots$. We assume that the action $a_k$ is sampled uniformly at random from all possible actions $\mathcal{A}$ (i.e., ``pure exploration"). Similarly to $(x,a)$, $(x_k, a_k) \in \mathbb{R}^d$ is a vector which is a concatenation of the vectors $x_k$ and $a_k$.

In this paper, we study the asymptotic behavior of the Q-network $Q^{N}(x,a; \theta_{k})$ as the number of hidden units $N$ and number of stochastic gradient descent iterates $k$ go to infinity. As we will see, after appropriate scalings, the Q-network converges to the solution of a limiting ODE. We then prove that limit ODE has a unique stationary solution which equals the solution to the Bellman equation. Thus, the unique stationary solution of the limit $Q$-network gives the optimal control for the problem. Convergence of ODEs to stationary solutions has been studied in related problems in the classical papers \cite{Borkar1998,Tsitsiklis1994}. The difference in our work is that, in contrast to  \cite{Borkar1998,Tsitsiklis1994}, we study the effect of a neural network as a function approximator in the Q-learning algorithm. Convergence to a global minimum for a neural network in regression or classification on i.i.d. data (not in the reinforcement learning setting) has been recently proven in \cite{NTK}, \cite{JasonLee}, \cite{Du1}, and \cite{UCLA2018}. Global convergence of reinforcement learning algorithms with neural network function approximations has also been recently studied in \cite{Cai2019}.

It is worthwhile noting that the Q-learning algorithm is similar to the stochastic gradient descent algorithm in that they both use stochastic samples to take training steps to minimize the objective function. However, unlike in the stochastic gradient descent algorithm (which we also discuss in Section \ref{S:SpecialCase}), the Q-learning update directions $G_k$ are not necessarily unbiased estimates of a descent direction for the objective function $L(\theta)$. The Q-learning algorithm calculates its update by taking the derivative of $L(\theta)$ \emph{while treating the target $Y$ as a constant}. Since $Y$ actually depends upon $\theta$,
\begin{eqnarray}
\mathbb{E} \big{[} G_k  |  \theta_k, x_k, a_k \big{]} \neq \frac{1}{2}\nabla_{\theta}  \left[\left(Y(x_k,a_k) -Q(x_k,a_k;\theta_k)\right)^{2}\right].
\label{NotUnbiased}
\end{eqnarray}

This fact together with the presence of the neural network function approximator leads to certain difficulties in the proofs. We will return to this issue in Remark \ref{R:ConvergenceInfiniteHorizon}.

Let us next present the main results of the paper in Section \ref{S:MainResults}.

\section{Main results}\label{S:MainResults}

In this section we present our main results. \textcolor{black}{In Subsection \ref{SS:InfTimeHorResults}, we discuss the infinite time horizon setting. In Subsection \ref{SS:fTimeHorResults}, we present our results for the finite time horizon setting. }

\subsection{The infinite time horizon setting}\label{SS:InfTimeHorResults}
 We start with the infinite time horizon setting. We consider the $Q$-network  (\ref{NetworkXavier}) which models the value of a state and action. The parameters $\theta$ for the $Q$-network are trained using the Q-learning algorithm (\ref{SGDupdates}). We prove that, as the number of hidden units and training steps become large, the $Q$-network converges in distribution to a random ordinary differential equation.

\begin{assumption} \label{A:Assumption2} Our results are proven under the following assumptions:
\begin{itemize}
\item The activation function $\sigma\in C^{2}_{b}(\mathbb{R})$, i.e. $\sigma$ is twice continuously differentiable and bounded.
\item The randomly initialized parameters $(C_0^i, W_0^i)$ are i.i.d., mean-zero random variables with a distribution $ \mu_0(dc, dw)$.  We assume that $\mu_0$ is absolutely continuous with respect to Lebesgue measure.
\item The random variable $C_0^i$ is bounded and $\la \norm{w}, \mu_0 \ra=\int \norm{w}\mu_0(dc, dw)< \infty$.
\item The reward function $r$ is uniformly bounded in its arguments.
\item The Markov chain $x_k$ has a limiting distribution $\pi$, namely \[\displaystyle  \pi(x)=\lim_{N\rightarrow\infty}\frac{1}{N} \sum_{i = 1}^N \mathbf{1}_{\{x_k = x | x_0=z\}} , \textrm{almost surely, for all initial states } z,\]
    exists, is independent of the initial state $z$, $\displaystyle \sum_{x\in\mathcal{X}}\pi(x)=1$ and $\pi(x) > 0$ for all $x \in \mathcal{X}$.
%    \item $\pi(x) > 0$ for all $x \in \mathcal{X}$.
\item $\mathcal{X}$ and $\mathcal{A}$ are finite, discrete spaces.
\end{itemize}
\end{assumption}

We shall also assume that the action $a\in\mathcal{A}$ is sampled uniformly at random from all possible actions (referred to as ``pure exploration"). The uniform distribution of the actions combined with the fact that $x_k$ is assumed to have a limiting distribution $\pi(x)$ imply that  the Markov chain $\zeta_k=(x_k,a_k)$ will have limiting distribution $\pi(x,a)=\frac{1}{K}\pi(x)$ where $K = | \mathcal{A} |$. A close examination of our proofs reveals that the assumption of ``pure exploration" is done without loss of generality and that the theory goes through for any (general) fixed policy for which the Markov chain $\zeta_k=(x_k,a_k)$ is ergodic such that $\pi(x,a)>0$ for any $(x,a)\in\mathcal{X}\times\mathcal{A}$.  In addition, the Markov chain $(x_{k+1},x_k,a_k)$ will have $\pi(x',x,a)=p(x'|x,a)\pi(x,a)$ as its limiting distribution.\footnote{There is a slight abuse of notation due to denoting all of these distributions with $\pi$. In our calculations, the specific limit distribution being used is clear via its argument $x$, $(x,a)$, or $(x', x, a)$.}

\begin{assumption} \label{A:AssumptionPositiveDefinite} Certain properties of the limit ODE also require the following assumptions:
\begin{itemize}
\item The activation function $\sigma$ is non-polynomial and slowly increasing\footnote{A function $\sigma(x)$ is called slowly increasing if $\lim_{x\rightarrow\infty}\frac{\sigma(x)}{x^{a}}=0$ for every $a>0$.}.
\item \textcolor{black}{$C_0^i \in [-B, B]$ and $\mu_0(\Gamma) > 0$ for any set $\Gamma \subset [-B,B] \times \mathbb{R}^{1+d}$ with positive Lebesgue measure.}
\end{itemize}
\end{assumption}

Examples of activation functions that are non-polynomials and slowly increasing are tanh and sigmoid activation function.

Define the empirical measure
\begin{align}
\nu_k^N = \frac{1}{N} \sum_{i=1}^N \delta_{C^i_k, W^i_k}.\label{Eq:EmpiricalMeasure}
\end{align}

In addition, let us set $Q^{N}_{k}(x,a)=Q^{N}(x,a; \theta_k)$ and define the scaled processes
\begin{align}
h_t^N(x,a) =& Q_{\floor*{N t}}^N(x,a), \notag \\
\mu_t^N =& \nu_{\floor*{N t}}^N.\label{Eq:ScaledPrelimitQuantities}
\end{align}

Using Assumption \ref{A:Assumption2}, we know that $\mu_0^N \overset{d} \rightarrow \mu_0$ and $h_0^N \overset{d} \rightarrow \mathcal{G}$ as $N \rightarrow \infty$ where $\mathcal{G}$ is a mean-zero Gaussian random variable, by the standard law of large numbers and central limit theorem for i.i.d random variables, respectively.

The variable $h_t^N$ is the output of the neural network after $\frac{t}{T} \times 100\%$ of the training has been completed. We will study convergence in distribution of the random process $(\mu_t^N, h_t^N)$ as $N \rightarrow \infty$ in the space $D_E([0,T])$ where $E = \mathcal{M}(\mathbb{R}^{1+d}) \times \mathbb{R}^{|\mathcal{X} \times \mathcal{A}|}$.  $D_E([0,T])$ is the Skorokhod space and $\mathcal{M}(S)$ is the space of probability measures on $S$.

Before presenting the first main convergence result, Theorem \ref{MainTheoremInfiniteTimeCase}, we present a lemma stating that a certain matrix $A$ which appears in the limit ODE is positive definite. The proof of Lemma \ref{PositiveDefiniteLemma} is deferred to Section \ref{ProofOFPositiveDefinite}.
\begin{lemma} \label{PositiveDefiniteLemma}
Let Assumption \ref{A:AssumptionPositiveDefinite} hold and let $0<\alpha<\infty$. Consider the matrix $A$ with elements
\begin{eqnarray*}
A_{\zeta, \zeta'} &= \alpha \la \sigma(w \cdot \zeta' ) \sigma(w \cdot \zeta ), \mu_0 \ra + \alpha\la c^2 \sigma'(w \cdot \zeta' ) \sigma'(w \cdot \zeta ) \zeta\cdot \zeta', \mu_0 \ra,
\end{eqnarray*}
for $\zeta, \zeta' \in \{ \zeta^{(1)}, \ldots, \zeta^{(M)} \}$ where $\zeta^{(i)} \in \mathcal{X}\times\mathcal{A}$ are assumed to be in distinct directions (as defined on page 192 of \cite{yIto}). Then,  the matrix $A$ is positive definite.
\end{lemma}

We then have the following theorem.
\begin{theorem} \label{MainTheoremInfiniteTimeCase}
Let Assumption \ref{A:Assumption2} hold and let the learning rate be $ \alpha^N =\frac{\alpha}{N}$ with $0<\alpha<\infty$. The process $(\mu_t^N, h_t^N)$ converges in distribution in the space $D_E([0,T])$ as $N \rightarrow \infty$ to $(\mu_t, h_t)$.  \textcolor{black}{For any $t\in[0,T]$, any $(x,a)\in\mathcal{X}\times\mathcal{A}$ and for every $f \in C^2_b( \mathbb{R}^{1+d})$, $(\mu_t, h_t)$ satisfies the random ODE}
\begin{align}
h_t(x,a) =& h_0(x,a) +  \int_0^t \sum_{(x', a') \in \mathcal{X}  \times \mathcal{A}} \pi(x', a') A_{x,a,x', a'}   \bigg{(} r(x', a') + \gamma \sum_{z \in \mathcal{X}} \max_{a'' \in \mathcal{A}} h_s(z, a'') p(z| x',a') - h_s(x', a') \bigg{)} ds \notag \\
h_0(x,a) =& \mathcal{G}(x,a), \notag \\
<f,\mu_t>  =& <f,\mu_0>.
\label{EvolutionEquationIntroductionXavierInfiniteTime}
\end{align}
The tensor $A$ is
\begin{align*}
A_{x, a, x', a'} &= \alpha \la \sigma(w \cdot (x', a') ) \sigma(w \cdot (x,a) ), \mu_0 \ra + \alpha \la c^2 \sigma'(w \cdot (x', a') ) \sigma'(w \cdot (x,a) ) (x',a')\cdot (x,a), \mu_0 \ra,
\end{align*}
and (\ref{EvolutionEquationIntroductionXavierInfiniteTime}) is a random ODE since its initial condition is a random variable. $\la f, \mu_t \ra = \la f, \mu_0 \ra$ for any $f \in C^2_b( \mathbb{R}^{1+d})$ implies that $\mu_t = \mu_0$ for $t \geq 0$ in the weak sense.

Furthermore, if Assumption \ref{A:AssumptionPositiveDefinite} holds, (\ref{EvolutionEquationIntroductionXavierInfiniteTime}) has a unique stationary point which equals the solution $V$ of the Bellman equation (\ref{HJBmain}).
\end{theorem}

The proof of Theorem \ref{MainTheoremInfiniteTimeCase} appears in Section \ref{ProofInfiniteCase}. Theorem \ref{MainTheoremInfiniteTimeCase} shows that there is a single fixed point for the limit dynamics of the $Q$-network. Moreover, this unique fixed point is the solution to the Bellman equation (\ref{HJBmain}) and therefore gives the optimal control. This is interesting since the pre-limit neural network is a non-convex function and therefore there are potentially many fixed points.

\textcolor{black}{At this point, a discussion of Theorem \ref{MainTheoremInfiniteTimeCase} is warranted. Theorem \ref{MainTheoremInfiniteTimeCase} says that $\mu_{t}=\mu_{0}$ in the weak sense and for all $t\geq 0$. This means that the limit of the empirical distribution of the parameters does not evolve as $t$ grows. However,  the limit of the network $h^{N}_{t}$ as defined by (\ref{Eq:ScaledPrelimitQuantities}) does evolve in time as the formula (\ref{EvolutionEquationIntroductionXavierInfiniteTime}) shows. This is due to the scaling by $1/\sqrt{N}$ of the neural network output (\ref{NetworkXavier}) as well as due to the training mechanism in (\ref{SGDupdates}).}

\begin{remark}
Theorem \ref{MainTheoremInfiniteTimeCase} suggests that in order to have convergence as $N\rightarrow\infty$, the learning rate $\alpha^{N}$ should be chosen to be of the order of $1/N$ where $N$ is the number of hidden units. \textcolor{black}{In practice, $N$ is fixed but large. This result shows that for $N$ large and if the learning rate is chosen to be of the order of $1/N$, then the neural network shall have the asymptotic behavior described in Theorem \ref{MainTheoremInfiniteTimeCase}.}
\end{remark}

Theorem \ref{MainTheoremInfiniteTimeCase} does not prove that $h_t$ converges to $V$. It only shows that, if $h_t$ converges to a limit point, that limit point must be $V$, the solution of the Bellman equation. We are able to prove convergence in the following lemma for small $\gamma$.

\begin{lemma}\label{L:ConvergenceInfiniteCase}
Let $\gamma < \frac{2}{1+K}$ where $K$ is the number of possible actions in the set $\mathcal{A}$. Then, we have
\[
\displaystyle \lim_{t \rightarrow \infty} \sup_{(x,a)\in\mathcal{X}\times\mathcal{A}}|h_t(x,a) - V(x,a)| = 0.
\]
\end{lemma}

\begin{remark}\label{R:ConvergenceInfiniteHorizon}
Convergence of ODEs of the type (\ref{EvolutionEquationIntroductionXavierInfiniteTime}) to solutions of the corresponding stationary equation has been studied in the literature in \cite{Borkar1998,Tsitsiklis1994}. The difference between our case and these earlier works is the nature of the matrix $A$ which appears in the ODE (e.g., see equation (\ref{EvolutionEquationIntroductionXavierInfiniteTime}) with the matrix $A$). In previous papers such as \cite{Borkar1998,Tsitsiklis1994}, the matrix $A$ is either an identity matrix or a diagonal matrix with diagonal elements uniformly bounded away from zero and with an upper bound of one. In our case, the Q-learning algorithm with a neural network produces an ODE with a matrix A that \emph{is not} a diagonal matrix. The arguments of \cite{Borkar1998,Tsitsiklis1994} do not establish convergence in the case where $A$ is non-diagonal. Lemma \ref{L:ConvergenceInfiniteCase} proves convergence for a non-diagonal matrix $A$ for small $\gamma$.

 Despite our best efforts we did not succeed in proving  Lemma \ref{L:ConvergenceInfiniteCase} for all $0<\gamma<1$ in our general case with the non-diagonal matrix A, which is produced by the neural network approximator in the Q-learning algorithm. As we discussed in Section \ref{S:BackgroundInformation}, the difficulties that arise here are also related to the fact that the Q-learning algorithm calculates its update by taking the derivative of $L(\theta)$ \emph{while treating the target $Y$ as a constant}. Hence, the asymptotic dynamics of the Q-network as $N$ and $k$ grow to infinity, which is the solution to the ODE (\ref{EvolutionEquationIntroductionXavierInfiniteTime}), may not necessarily move in the descent direction of the limiting objective function (this is in contrast to the standard regression problem with i.i.d. data that we study in Section \ref{S:SpecialCase}). \textcolor{black}{It is also worthwhile noting that the global convergence results of \cite{Cai2019} are also valid for sufficiently small $\gamma$. }

\subsection{The finite time horizon setting}\label{SS:fTimeHorResults}

In contrast to the infinite horizon setting, and as shown in Theorem \ref{MainTheorem1} below, one can prove convergence for all values of the discount factor $\gamma\in(0,1]$ in the finite time horizon case. We are able to prove convergence for all $\gamma \in (0,1]$ because in the finite time horizon case one can study the large time limit of the limiting ODE recursively.
\end{remark}

We now consider the finite time horizon problem. The $Q$-network, which models the value of state $x$ and action $a$ at time $j$, is
\begin{align*}
Q(j,x,a) = \frac{1}{\sqrt{N}} \sum_{i=1}^N C^{i,j} \sigma \big{(} W^i \cdot (x,a) \big{)},
\end{align*}
where $C^i \in \mathbb{R}^J$, $W^i \in \mathbb{R}^d$, $d = d_x + d_a$, and $\sigma: \mathbb{R} \rightarrow \mathbb{R}$. Note that the parameter $W^i$ is shared across all times $j$.

The model parameters $\theta$ are trained using the Q-learning algorithm. The parameter updates are given by, for training iterations $k = 0, 1, \ldots$ and times $j = 0, \ldots, J-1$, the following equations:
\begin{align}
C^{i,j}_{k+1} =& C^{i,j}_k + \frac{ \alpha^N }{\sqrt{N}} \bigg{(} r_{k,j} + \gamma \max_{a' \in \mathcal{A}} Q^N(j+1,x_{k,j+1}, a'; \theta_k) - Q^N(j,x_{k,j}, a_{k,j}; \theta_k) \bigg{)} \sigma \big{(} W^i_k \cdot (x_{k,j}, a_{k,j}) \big{)}, \notag \\
W^i_{k+1} =& W^i_k + \frac{ \alpha^N}{\sqrt{N}} \sum_{j=0}^{J-1}  \bigg{(} r_{k,j} + \gamma \max_{a' \in \mathcal{A}} Q^N(j+1, x_{k,j+1}, a'; \theta_k) - Q^N(j, x_{k,j}, a_{k,j}; \theta_k) \bigg{)}  \notag \\
\times& C^{i,j}_k \sigma' \big{(} W^i_k (x_{k,j}, a_{k,j}) \big{)} (x_{k,j}, a_{k,j}), \notag \\
Q^{N}(j, x,a; \theta_k) =& \frac{1}{\sqrt{N}} \sum_{i=1}^N C^{i,j}_k \sigma \big{(} W^i_k  \cdot (x,a) \big{)}, \label{SGDupdates00}
\end{align}
where $r_{k,j} = r(j, x_{k,j}, a_{k,j})$ and $x_{k,j+1}\sim p(\cdot |x_{k,j}, a_{k,j})$ and $a_{k,j}\sim \text{Uniform}(1/K)$. For notational convenience, define $Q^{N}_k(j, x,a) = Q^{N}(j, x,a; \theta_k)$ and $Q^N(J, x, a; \theta_k) = r(J, x)$.%We let the learning rate $\alpha^N = \frac{1}{N}$.

\begin{assumption} \label{A:AssumptionPositiveDefiniteFiniteTime} Certain properties of the limit ODE also require the following assumptions:
\begin{itemize}
\item The activation function $\sigma$ is non-polynomial and slowly increasing.
\item \textcolor{black}{$C_0^i \in [-B, B]^J$ and $\mu_0(\Gamma) > 0$ for any set $\Gamma \subset [-B,B]^J \times \mathbb{R}^d$ with positive Lebesgue measure.}
\end{itemize}
\end{assumption}

Similar to the infinite time horizon case, we define the processes $\nu_k^{N,j} = \frac{1}{N} \sum_{i=1}^N \delta_{C^{i,j}_k, W^i_k}$, $\mu_t^{N,j} = \nu_{\floor*{N t}}^{N,j}$, and $h_t^{N}(j,x,a)= Q_{\floor*{N t}}^N(j,x,a)$. We will study convergence in distribution of the random process $(\mu_t^N, h_t^N)$ as $N \rightarrow \infty$ in the space $D_E([0,T])$ where $E = \mathcal{M}(\mathbb{R}^{J+d}) \times \mathbb{R}^{|\mathcal{X} \times \mathcal{A}|}$. Denote the probability distribution of $(x_{k,j}, a_{k,j})$ denoted as $\pi_j(x_k,a_k)$.  We then have the  following theorem.
\begin{theorem} \label{MainTheorem1}
Let Assumption \ref{A:Assumption2} hold and let the learning rate be $\alpha^N = \frac{\alpha}{N}$ with $0<\alpha<\infty$. In addition, we assume that for $j=0,1,2,\cdots,J$ the stationary distribution $\pi_{j}(x,a)>0$ for every $(x,a)\in\mathcal{X}\times\mathcal{A}$. The process $(\mu_t^N, h_t^N)$ converges in distribution in the space $D_E([0,T])$ as $N \rightarrow \infty$ to $(\mu_t, h_t)$. \textcolor{black}{For any $t\in[0,T]$, $j=1,2,\cdots, J$, $(x,a)\in\mathcal{X}\times\mathcal{A}$ and for every $f \in C_2^b( \mathbb{R}^{1+d})$, $(\mu_t^N, h_t^N)$ satisfies the random ODE}
\begin{align}
h_t(j, x,a) =&  h_0(j,x,a) +  \int_0^t \sum_{(x', a') \in \mathcal{X} \times \mathcal{A}} \pi_j(x', a') A_{x,a,x', a'}   \bigg{(} r(j, x', a') \notag \\
+& \gamma \sum_{z} \max_{a'' \in \mathcal{A}} h_s(j+1,z, a'') p(z | x', a') - h_s(j, x', a') \bigg{)} ds, \notag \\
h_0(j, x,a) =& \mathcal{G}(j, x, a), \notag \\
<f,\mu_t> =& <f,\mu_0>,
\label{EvolutionEquationIntroductionXavierFiniteTimeHorizon}
\end{align}
where $j = 0, 1, \ldots, J-1$ and $h_t(J,x,a) = r(J, x)$. The tensor $A$ is
\begin{align}
A_{x, a, x', a'} &= \alpha \la \sigma(w \cdot (x', a') ) \sigma(w \cdot (x,a) ), \mu_0 \ra + \alpha\la c_j^2 \sigma'(w \cdot (x', a') ) \sigma'(w \cdot (x,a) ) (x',a')\cdot (x,a), \mu_0 \ra.\notag
\end{align}

Furthermore, if Assumption \ref{A:AssumptionPositiveDefiniteFiniteTime} holds, the neural network converges to the solution of the Bellman equation (\ref{HJBmainFiniteT}):
\begin{align}
\lim_{t \rightarrow \infty} \sup_{j,x,a}|h_t(j,x,a) - V(j,x,a)| = 0.\notag
\end{align}

\end{theorem}
\begin{proof}
The proof of this result is in Section \ref{S:ProofFiniteCase}.
\end{proof}

\section{A special case: neural networks and regression}\label{S:SpecialCase}

The asymptotic approach developed in this paper can be used to study other popular cases in machine learning. For example, consider the case of the objective function (\ref{Eq:LossFunction}) but with $y_{k}$ now being independent samples from a fixed distribution. Then, (\ref{Eq:LossFunction}) is simply the mean-squared error objective function for regression. Using the same techniques as we employ on (the more difficult) Q-learning problem discussed in the previous section, we can establish the asymptotic behavior of neural network models used in regression. Let the neural network be
\begin{align}
g^{N}(x; \theta) = \frac{1}{\sqrt{N}} \sum_{i=1}^N C^i \sigma( W^i \cdot  x),\nonumber
%\label{NetworkXavier}
\end{align}
where $C^i \in \mathbb{R}$, $W^i \in \mathbb{R}^d$, $x \in \mathbb{R}^d$, and $\sigma(\cdot): \mathbb{R} \rightarrow \mathbb{R}$. The objective function is
\begin{align}
\mathcal{L}^{N}(\theta) = \mathbb{E} \bigg{[} (Y - g^N(X; \theta) )^2 \bigg{]},\nonumber
\end{align}
where the data $(X,Y) \sim \pi(dx,dy)$, $Y \in \mathbb{R}$, and the parameters $\theta = (C^1, \ldots, C^N, W^1, \ldots, W^N ) \in \mathbb{R}^{N \times (1 + d) }$. %For notational convenience, we may refer to $g^N(x; \theta)$ as $g^N(x)$ in our analysis below.

The model parameters $\theta$ are trained using stochastic gradient descent. The parameter updates are given by:
\begin{align}
C^i_{k+1} =& C^i_k + \frac{ \alpha^N }{\sqrt{N}} ( y_k - g_k^N(x_k) ) \sigma(W^i_k \cdot x_k ), \notag \\
W^i_{k+1} =& W^i_k + \frac{ \alpha^N}{\sqrt{N}}  ( y_k - g_k^N(x_k) ) C^i_k \sigma'(W^i_k \cdot x_k ) x_k, \notag \\
g_k^N(x) =&  \frac{1}{\sqrt{N}} \sum_{i=1}^N C^i_k \sigma( W^i_k \cdot x),\notag
%\label{SGDupdatesXavier}
\end{align}
for $k = 0, 1, \ldots$. $\alpha^N$ is the learning rate (which may depend upon $N$). The data samples are $(x_k, y_k)$ are i.i.d. samples from a distribution $\pi(dx,dy)$.

In Theorem \ref{MainTheorem3} we prove that a neural network with the Xavier initialization (i.e., with the $\frac{1}{\sqrt{N}}$ normalization in the formula for $g^{N}(x;\theta)$) and trained with stochastic gradient descent converges in distribution to a random ODE as the number of units and training steps become large. Although the pre-limit problem of optimizing a neural network is non-convex (and therefore the neural network may converge to a local minimum), the limit equation minimizes a quadratic objective function. In Theorem \ref{ZeroTrainingError}, we also show that  the neural network (in the limit) will converge to a global minimum with zero loss on the training set. Convergence to a global minimum for a neural network has been recently proven in \cite{JasonLee}, \cite{Du1}, and \cite{UCLA2018}. Our result shows that convergence to a global minimum can also be viewed as a simple consequence of the mean-field limit for neural networks.

For completeness, we also mention here that other scaling regimes have also been studied in the literature. In particular, \cite{Chizat2018,Montanari,RVE,SirignanoSpiliopoulosNN1,SirignanoSpiliopoulosNN2} study the asymptotics of single-layer neural networks with a $\frac{1}{N}$ normalization; that is, $g^{N}(x; \theta) = \frac{1}{N} \sum_{i=1}^N C^i \sigma( W^i  \cdot x)$. \cite{SirignanoSpiliopoulosNN3} studies the asymptotics of deep (i.e., multi-layer) neural networks with a $\frac{1}{N}$ normalization in each hidden layer. The $\frac{1}{N}$ normalization is convenient since the single-layer neural network is then in a traditional mean-field framework where it can be described via an empirical measure of the parameters. In the single layer case, the limit for the neural network satisfies a partial differential equation. As discussed in \cite{Montanari}, it is \emph{not} necessarily true that the limiting equation (a PDE in this case) will converge to the global minimum of an objective function with zero training error (i.e., zero loss on the training dataset). However, the $\frac{1}{\sqrt{N}}$ normalization that we study in this paper is more widely-used in practice (see \cite{Xavier}) and, importantly, as we demonstrate in Theorem \ref{ZeroTrainingError}, the limit equation converges to a global minimum with zero training error.

Lastly, we mention here that \cite{NTK} proved, using different methods, a limit for neural networks with a $\frac{1}{\sqrt{N}}$ Xavier initialization when they are trained with continuous-time gradient descent. Our result in Theorem \ref{MainTheorem3} proves a limit for neural networks trained with the (standard) discrete-time stochastic gradient descent algorithm which is used in practice, and rigorously passes from discrete time (where the stochastic gradient descent updates evolve) to continuous time.

\begin{assumption} \label{A:Assumption3} We impose the following assumptions:
\begin{itemize}
\item The activation function $\sigma\in C^{2}_{b}(\mathbb{R})$, i.e. $\sigma$ is twice continuously differentiable and bounded.

\item The randomly initialized parameters $(C_0^i, W_0^i)$ are i.i.d., mean-zero random variables with a distribution $ \mu_0(dc, dw)$.
\item The random variable $C_0^i$ has compact support and $\la \norm{w}, \mu_0 \ra < \infty$.
\item %The sequence of data samples $(x_k, y_k)$ is i.i.d. from the probability distribution $\pi(dx,dy)$.
    There is a fixed dataset of $M$ data samples $(x^{(i)}, y^{(i)} )_{i=1}^M$ and we set $\pi(dx,dy) = \displaystyle \frac{1}{M} \sum_{i=1}^M \delta_{(x^{(i)}, y^{(i)}) }(dx,dy)$.
%\item there is a fixed dataset of $M$ data samples $(x^{(i)}, y^{(i)} )_{i=1}^M$ and therefore $\pi(dx,dy) = \displaystyle \frac{1}{M} \sum_{i=1}^M \delta_{(x^{(i)}, y^{(i)}) }(dx,dy)$.
\end{itemize}
\end{assumption}
Note that the last assumption also implies that $\pi(dx,dy)$ has compact support.

Following the asymptotic procedure developed in this paper, we can study the limiting behavior of the network output $g^N_k(x)=g^{N}(x;\theta_k)$ for $x \in \mathcal{D} = \{ x^{(1)}, \ldots, x^{(M)} \}$ as the number of hidden units $N$ and stochastic gradient descent steps $k$ simultaneously become large, after appropriately relating $k$ and $N$. The network output converges in distribution to the solution of a random ODE as $N \rightarrow \infty$.

For this purpose, let us recall the empirical measure defined in (\ref{Eq:EmpiricalMeasure}). Note that the neural network output can be written as the inner-product
\begin{align}
g_k^N(x) =  \la c \sigma(w \cdot x), \sqrt{N} \nu_k^N \ra.\nonumber
\end{align}

Due to Assumption \ref{A:Assumption3}, as $N \rightarrow \infty$ and for $x \in \mathcal{D}$,
\begin{align}
g_0^N(x) \overset{d} \rightarrow \mathcal{G}(x),
\label{Gdefinition}
\end{align}
where $\mathcal{G} \in \mathbb{R}^M$ is a Gaussian random variable. We also of course have that
\begin{align}
\nu_0^N \overset{p} \rightarrow \nu_0 \equiv \mu_0.\nonumber
\end{align}

Define the scaled processes
\begin{align}
h_t^N =& g_{\floor*{N t}}^N, \notag \\
\mu_t^N =& \nu_{\floor*{N t}}^N,\nonumber
\end{align}
where $g_k^N = \bigg{(} g_k^N(x^{(1)}), \ldots, g_k^N(x^{(M)}) \bigg{)}$, $h_t^N(x) = g_{\floor*{N t}}^N(x)$, and $h_t^N = \bigg{(} h_t^N(x^{(1)}), \ldots, h_t^N(x^{(M)}) \bigg{)}$.

Now, we are ready to state the main result of this section,  Theorem \ref{MainTheorem3}.
\begin{theorem} \label{MainTheorem3}
Let Assumption \ref{A:Assumption3} hold, set $\alpha^{N}= \frac{\alpha}{N}$ with $0<\alpha<\infty$ and define $E = \mathcal{M}(\mathbb{R}^{1+d}) \times \mathbb{R}^{M}$. The process $(\mu_t^N, h_t^N)$ converges in distribution in the space $D_E([0,T])$ as $N \rightarrow \infty$ to $(\mu_t, h_t)$ which satisfies, for every $f \in C_2^b( \mathbb{R}^{1+d})$, the random ODE
\begin{align}
h_t(x) =& h_0(x) +   \alpha \int_0^t \int_{\mathcal{X} \times \mathcal{Y}} ( y - h_s(x') ) \la \sigma(w \cdot x ) \sigma(w \cdot x'), \mu_s \ra  \pi(dx',dy) ds \notag \\
+& \alpha \int_0^t \int_{\mathcal{X} \times \mathcal{Y}} ( y - h_s(x') )   \la c^2  \sigma'(w \cdot x')  \sigma'(w \cdot  x) x\cdot x', \mu_s \ra \pi(dx',dy) ds, \notag \\
h_0(x) =& \mathcal{G}(x), \notag \\
\la f, \mu_t \ra =& \la f, \mu_0 \ra.
\label{EvolutionEquationIntroductionXavier}
\end{align}

\end{theorem}
\begin{proof}
The proof of this theorem is omitted because it is exactly parallel to the proof of Theorem \ref{MainTheoremInfiniteTimeCase}.
\end{proof}

Recall that $\mathcal{G} \in \mathbb{R}^M$ is a Gaussian random variable; see equation (\ref{Gdefinition}). In addition, note that $\bar \mu_t$ in the limit equation (\ref{EvolutionEquationIntroductionXavier}) is a constant, i.e. $\mu_t = \mu_0$ for $t \in [0,T]$. Therefore, (\ref{EvolutionEquationIntroductionXavier}) reduces to
\begin{align}
h_t(x) =& h_0(x) +   \alpha \int_0^t \int_{\mathcal{X} \times \mathcal{Y}} ( y - h_s(x') ) \la \sigma(w \cdot x ) \sigma(w \cdot x'), \mu_0 \ra  \pi(dx',dy) ds \notag \\
+& \alpha \int_0^t \int_{\mathcal{X} \times \mathcal{Y}} ( y - h_s(x') )   \la c^2  \sigma'(w \cdot x')  \sigma'(w \cdot x) x\cdot x', \mu_0 \ra \pi(dx',dy) ds, \notag \\
h_0(x) =& \mathcal{G}(x).
\label{LimitEquation2}
\end{align}

Since (\ref{LimitEquation2}) is a linear equation in $C_{\mathbb{R}^M}([0,T])$, the solution $h_t$ is unique. \\

To better understand (\ref{LimitEquation2}), define the matrix $A \in \mathbb{R}^{M \times M}$ where
\begin{align*}
A_{x, x'} = \frac{\alpha}{M} \la \sigma(w \cdot x ) \sigma(w \cdot x'), \mu_0 \ra   +  \frac{\alpha}{M}    \la c^2  \sigma'(w \cdot x')  \sigma'(w \cdot x) x\cdot x', \mu_0 \ra,%\label{Eq:AmatrixXavier}
\end{align*}
where $x, x' \in \mathcal{D}$. $A$ is finite-dimensional since we fixed a training set of size $M$ in the beginning.

Then, (\ref{LimitEquation2}) becomes
\begin{align}
d h_t =& A  \bigg{(} \hat Y -  h_t \bigg{)} dt, \notag \\
h_0 =& \mathcal{G},\nonumber
%\label{LimitEquation2}
\end{align}
where $\hat Y = ( y^{(1)}, \ldots,  y^{(M)} )$.

Therefore, $h_t$ is the solution to a continuous-time gradient descent algorithm which minimizes a quadratic objective function.
\begin{align}
\frac{d h_t}{d t} =& - \frac{1}{2} \nabla_h J(\hat Y, h_t), \notag \\
J( y, h) =& \big{(} y -  h \big{)}^{\top} A \big{(} y -  h \big{)},  \notag \\
h_0 =&  \mathcal{G}.
\nonumber%\label{LimitEquation3}
\end{align}

Therefore, even though the pre-limit optimization problem is non-convex, the neural network's limit will minimize a quadratic objective function.

An interesting question is whether $h_t \rightarrow \hat Y$ as $t \rightarrow \infty$. That is, in the limit of large numbers of hidden units and many training steps, does the neural network model converge to a global minimum with zero training error? Theorem \ref{ZeroTrainingError} shows that $h_t \rightarrow \hat Y$ as $t \rightarrow \infty$ if $A$ is positive definite. Lemma \ref{PositiveDefiniteLemma} proves that, under reasonable hyperparameter choices and if the data samples are distinct, $A$ will be positive definite. \\

\begin{theorem} \label{ZeroTrainingError}
If Assumption \ref{A:AssumptionPositiveDefinite} holds and the data samples are distinct, then
\begin{align}
h_t \rightarrow \hat Y \phantom{....} \textrm{as} \phantom{....} t \rightarrow \infty.\nonumber
\end{align}

\end{theorem}
\begin{proof}
Consider the transformation $\tilde h_t = h_t - \hat Y$. Then,
\begin{align}
d \tilde h_t =& - A \tilde h_t dt, \notag \\
\tilde h_0 =& \mathcal{G} - \hat Y.
\nonumber%\label{LimitEquation2}
\end{align}
Then, $\tilde h_t \rightarrow 0$ (and consequently $h_t \rightarrow \hat Y$) as $t \rightarrow \infty$ if $A$ is positive definite. Lemma \ref{PositiveDefiniteLemma} proves that $A$ is positive definite under the Assumption \ref{A:AssumptionPositiveDefinite} and if the data samples are distinct.

\end{proof}

In connection to Theorem \ref{ZeroTrainingError}, we note that the data samples in the dataset will be distinct with probability $1$ if the random variable $X$ has a probability density function.

\section{Proof of Convergence in Infinite time Horizon Case} \label{ProofInfiniteCase}

We prove Theorem \ref{MainTheoremInfiniteTimeCase} in this section. The proof is divided into three parts. Let $\rho^N$ be the probability measure of a convergent subsequence of $\left(\mu^N, h^N \right)_{0\leq t\leq T}$. In Section \ref{PreLimitEvolution} we write the prelimit in a form that is convenient in order to establish the desired limiting behavior.   In Section \ref{RelativeCompactness}, we prove that the sequence $\rho^N$ is relatively compact (which implies that there is a subsequence $\rho^{N_k}$ which weakly converges). In Section \ref{Identification}, we prove that any limit point of $\rho^N$ is a probability measure of the random ODE (\ref{EvolutionEquationIntroductionXavierInfiniteTime}). In \ref{Uniqueness}, we prove that the limit point is unique. These three results are collected together in Section \ref{ProofOfConvergence} to prove that $(\mu^N, h^N)$ converges in distribution to $(\mu, h)$.

\subsection{Evolution of the Pre-limit Process} \label{PreLimitEvolution}

For notational convenience, let $Q^N(x,a; \theta_k) = Q_k^N(x,a)$, $\zeta = (x,a)$, and $\zeta_k = (x_k, a_k)$. We study the evolution of $Q_k^N(x,a)$ during training.
\begin{align}
Q_{k+1}^N(\zeta)  =& Q_{k}^N(\zeta)  + \frac{1}{\sqrt{N}} \sum_{i=1}^N C^{i}_{k+1} \sigma( W^i_{k+1} \cdot \zeta ) - \frac{1}{\sqrt{N}} \sum_{i=1}^N C^{i}_k \sigma( W^i_k \cdot \zeta ) \notag \\
=& Q_{k}^N(\zeta)  + \frac{1}{\sqrt{N}} \sum_{i=1}^N \bigg{(} C^{i}_{k+1} \sigma( W^i_{k+1} \cdot \zeta) -  C^{i}_k \sigma( W^i_k \cdot \zeta) \bigg{)} \notag \\
=& Q_{k}^N(\zeta)  + \frac{1}{\sqrt{N}} \sum_{i=1}^N \bigg{(} ( C^{i}_{k+1} - C^{i}_k )  \sigma( W^i_{k+1} \cdot \zeta)  + (  \sigma( W^i_{k+1} \cdot \zeta) -   \sigma( W^i_k \cdot \zeta)  ) C^{i}_k \bigg{)} \notag \\
=& Q_{k}^N(\zeta)  + \frac{1}{\sqrt{N}} \sum_{i=1}^N \bigg{(} ( C^{i}_{k+1} - C^{i}_k )  \bigg{[} \sigma(W^i_k \cdot \zeta) +  \sigma'( W^{i,\ast}_{k} \cdot \zeta ) \zeta\cdot ( W^i_{k+1} - W^i_{k} ) \bigg{]}   \notag \\
+& \bigg{[} \sigma'(W^i_k \cdot \zeta) \zeta \cdot ( W^i_{k+1} - W^i_k )   +  \frac{1}{2}\sigma''( W^{i, \ast \ast}_{k} \zeta) \left( ( W^{i}_{k+1} - W^{i}_{k} )\cdot \zeta \right)^2  \bigg{]} C^{i}_k \bigg{)},
\label{gEvolution1}
\end{align}
for points $W^{i,\ast}_{k}$ and $W^{i,\ast,\ast}_{k}$ in the line segment connecting the points $W^{i}_{k}$ and $W^{i}_{k+1}$. Let $\alpha^N = \frac{\alpha}{N}$. Substituting (\ref{SGDupdates}) into (\ref{gEvolution1}) yields

\begin{align}
Q_{k+1}^N(\zeta)  =& Q_{k}^N(\zeta) + \frac{\alpha}{N^2}  \bigg{(} r_k + \gamma \max_{a' \in \mathcal{A}} Q^N_k(x_{k+1}, a') - Q^N_k(\zeta_k) \bigg{)} \sum_{i=1}^N \sigma(W^i_k \cdot \zeta_k ) \sigma(W^i_k \cdot \zeta)   \notag \\
+& \frac{\alpha}{N^2}   \bigg{(} r_k + \gamma \max_{a' \in \mathcal{A}} Q^N_k(x_{k+1}, a') - Q^N_k(\zeta_k) \bigg{)}  \sum_{i=1}^N \sigma'(W^i_k \cdot \zeta)    \sigma'(W^i_k \cdot \zeta_k ) \zeta_k \cdot \zeta  (C^{i}_k)^2  \notag \\
+& \mathcal{O}_{p}(N^{-3/2}),
\label{gEvolution1b}
\end{align}
where when we write $Z_N=\mathcal{O}_{p}(\beta_N)$ for a random variable $Z_{N}$ we mean that $Z_{N}/\beta_{N}$ is stochastically bounded, in that for every $\epsilon>0$, there exists $M<\infty$ and some $N_{0}<\infty$ such that $\mathbb{P}\left(|Z_{N}/\beta_{N}|>M\right)<\epsilon$ for every $N>N_{0}$.

We can re-write the evolution of $Q_k^N(\zeta)$ in terms of the empirical measure $\nu_k^N$,
\begin{align}
Q_{k+1}^N(\zeta)  =& Q_{k}^N(\zeta) + \frac{\alpha}{N}  \bigg{(} r_k + \gamma \max_{a' \in \mathcal{A}} Q^N_k(x_{k+1}, a') - Q^N_k(\zeta_k) \bigg{)}  \la \sigma(w \cdot \zeta_k ) \sigma(w \cdot \zeta), \nu_k^N \ra   \notag \\
+&  \frac{\alpha}{N}   \bigg{(} r_k + \gamma \max_{a' \in \mathcal{A}} Q^N_k(x_{k+1}, a') - Q^N_k(\zeta_k) \bigg{)}  \la \sigma'(w \cdot \zeta)    \sigma'(w  \cdot \zeta_k ) \zeta_k\cdot \zeta  c^2, \nu_k^N \ra  \notag \\
+& \mathcal{O}_{p}(N^{-3/2}).
\label{gEvolution2}
\end{align}

Using (\ref{gEvolution2}), we can write the evolution of $h_t^N(\zeta)$ for $t \in [0,T]$ as
\begin{align}
h_t^N(\zeta) =& h_0^N(\zeta) + \sum_{k=0}^{\floor*{N t}-1} (  Q_{k+1}^N(\zeta) - Q_k^N(\zeta) ) \notag \\
=& h_0^N(\zeta) +\frac{\alpha}{N}  \sum_{k=0}^{\floor*{N t}-1}   \bigg{(} r_k + \gamma \max_{a' \in \mathcal{A}} Q^N_k(x_{k+1}, a') - Q^N_k(\zeta_k) \bigg{)}  \la \sigma(w \cdot \zeta_k ) \sigma(w \cdot \zeta), \nu_k^N \ra   \notag \\
+&   \frac{\alpha}{N}   \sum_{k=0}^{\floor*{N t}-1}  \bigg{(} r_k + \gamma \max_{a' \in \mathcal{A}} Q^N_k(x_{k+1}, a') - Q^N_k(\zeta_k) \bigg{)}  \la \sigma'(w \cdot \zeta)    \sigma'(w  \cdot \zeta_k ) \zeta_k\cdot \zeta  c^2, \nu_k^N \ra  \notag \\
+& \mathcal{O}_{p}(N^{-1/2}) \notag
\end{align}

This can then be rewritten as follows
\begin{align}
h_t^N(\zeta) =& h_0^N(\zeta) + \sum_{k=0}^{\floor*{N t}-1} (  Q_{k+1}^N(\zeta) - Q_k^N(\zeta) ) \notag \\
=& h_0^N(\zeta) + \alpha \int_0^t \sum_{(\zeta', x'') \in \mathcal{X} \times \mathcal{A}  \times \mathcal{X}}     \bigg{(} r(\zeta') + \gamma \max_{a'' \in \mathcal{A}} h_s^N(x'', a'')- h_s^N(\zeta') \bigg{)} \notag \\
& \times \la \sigma \big{(} w \cdot  \zeta' \big{)} \sigma \big{(} w \cdot \zeta \big{)}, \mu_s^N \ra \pi(x'', \zeta')  ds   \notag \\
+&  \alpha   \int_0^t  \sum_{(\zeta', x'') \in \mathcal{X} \times \mathcal{A}  \times \mathcal{X}}  \bigg{(} r(\zeta') + \gamma \max_{a'' \in \mathcal{A}} h_s^N(x'', a'') - h_s^N(\zeta') \bigg{)}  \notag \\
& \times \la c^2 \sigma' \big{(} w \cdot \zeta \big{)}    \sigma' \big{(}w  \cdot \zeta' \big{)} \zeta'\cdot \zeta , \mu_s^N \ra \pi(x'', \zeta') ds \notag \\
+& M_t^{1,N} + M_t^{2,N} + M_t^{3,N} +  \mathcal{O}_{p}(N^{-1/2}),\label{Eq:h_equationPrelimit}
\end{align}
where $\pi(x'', \zeta') = p(x'' | \zeta') \pi(\zeta')$. The fluctuation terms are
\begin{align}
M_t^{1,N}(\zeta) =& -\frac{1}{N}  \sum_{k=0}^{\floor*{N t}-1} Q^N_k(\zeta_k) B_{\zeta,\zeta_k,k}^N   + \frac{1}{N}  \sum_{k=0}^{\floor*{N t}-1} \sum_{\zeta' \in \mathcal{X} \times \mathcal{A}}   Q_k^N(\zeta')  B_{\zeta,\zeta',k}^N \pi(\zeta') , \notag \\
M_t^{2,N}(\zeta) =& \frac{1}{N}  \sum_{k=0}^{\floor*{N t}-1} r_k B_{\zeta,\zeta_k,k}^N   - \frac{1}{N}  \sum_{k=0}^{\floor*{N t}-1}  \sum_{\zeta' \in \mathcal{X} \times \mathcal{A}}  r(\zeta') B_{\zeta,\zeta',k}^N   \pi( \zeta') , \notag \\
M_t^{3,N}(\zeta) =& \frac{1}{N}  \sum_{k=0}^{\floor*{N t}-1} \gamma \max_{a'' \in \mathcal{A}} Q^N_k(x_{k+1}, a'')  B_{\zeta,\zeta_k,k}^N   \notag \\
& -  \frac{1}{N}  \sum_{k=0}^{\floor*{N t}-1}  \sum_{(\zeta', x'') \in \mathcal{X} \times \mathcal{A} \times \mathcal{X} }  \gamma \max_{a'' \in \mathcal{A}} Q_k^N(x'', a'') B_{\zeta,\zeta',k}^N \pi(x'', \zeta') , \label{Eq:h_equationPrelimit2}
\end{align}
where
\begin{align}
B_{\zeta,\zeta',k}^N =& \alpha \bigg{(} \la \sigma \big{(} w \cdot \zeta' \big{)} \sigma \big{(} w \cdot \zeta \big{)}, \nu_k^N \ra   +  \la \sigma' \big{(} w \cdot \zeta \big{)}    \sigma' \big{(} w  \cdot \zeta' \big{)} \zeta'\cdot \zeta  c^2, \nu_k^N \ra \bigg{)}.\label{Eq:B_term}
\end{align}

Later on, in Lemma \ref{Mlemma}, we prove that that the fluctuation terms  $M_t^{i,N}(\zeta)$ go to zero in $L^1$ as $N\rightarrow\infty$.

The evolution of the empirical measure $\nu_k^N$ can be characterized in terms of their projection onto test functions $f \in C^2_b(\mathbb{R}^{1 + d })$. A Taylor expansion yields
\begin{align}
\la f , \nu^N_{k+1} \ra - \la f , \nu^N_k \ra %=&  \frac{1}{N} \sum_{i=1}^N \bigg{(} \la f, \nu_{k+1}^N \ra - \la f, \nu_{k}^N \ra \bigg{)} \notag \\
=&  \frac{1}{N} \sum_{i=1}^N \bigg{(} f(C^i_{k+1}, W^i_{k+1} ) -  f(C^i_{k}, W^i_{k} )  \bigg{)} \notag \\
=& \frac{1}{N} \sum_{i=1}^N \partial_c f(C^i_{k}, W^i_{k} ) ( C^i_{k+1} -  C^i_{k} )  + \frac{1}{N} \sum_{i=1}^N \partial_w  f(C^i_{k}, W^i_{k} )  \cdot ( W^i_{k+1} -  W^i_{k} ) \notag \\
+& \frac{1}{2 N} \sum_{i=1}^N \partial^{2}_{c} f(\bar C^i_{k},  \bar W^i_{k} ) ( C^i_{k+1} -  C^i_{k} )^2  + \frac{1}{2 N} \sum_{i=1}^N ( C^i_{k+1} -  C^i_{k} )\partial_{cw}  f(\hat C^i_{k}, \hat W^i_{k} )( W^i_{k+1} -  W^i_{k} )    \notag \\
+& \frac{1}{2 N} \sum_{i=1}^N ( W^i_{k+1} -  W^i_{k} )^{\top}\partial^{2}_{w} f(\tilde C^i_{k}, \tilde W^i_{k} ) ( W^i_{k+1} -  W^i_{k} ),
\label{NuEvolution1}
\end{align}
for points $(\bar C^{i}_{k}, \bar W^{i}_{k})$, $(\hat C^{i}_{k}, \hat W^{i}_{k})$ and $(\tilde C^{i}_{k}, \tilde W^{i}_{k})$ in the segments connecting $C^i_{k+1}$ with $C^i_{k}$ and  $W^i_{k+1}$ with $W^i_{k}$, respectively.

Substituting (\ref{SGDupdates}) into (\ref{NuEvolution1}) yields
\begin{align}
\la f , \nu^N_{k+1} \ra - \la f , \nu^N_k \ra &=  N^{-5/2} \sum_{i=1}^N \partial_c f(C^i_{k}, W^i_{k} ) \alpha  \bigg{(} r_k + \gamma \max_{a' \in \mathcal{A}} Q^N_k(x_{k+1}, a') - Q^N_k(\zeta_k) \bigg{)} \sigma \big{(} W^i_k \cdot \zeta_k \big{)} \notag \\
&+ N^{-5/2} \sum_{i=1}^N  \partial_w  f(C^i_{k}, W^i_{k} )    \alpha \bigg{(} r_k + \gamma \max_{a' \in \mathcal{A}} Q^N_k(x_{k+1}, a') - Q^N_k(\zeta_k) \bigg{)} C^{i}_k \sigma' \big{(} W^i_k \cdot \zeta_k \big{)} (\zeta_k)  + O_{p}\left(N^{-2}\right) \notag \\
&=  N^{-3/2} \alpha  \bigg{(} r_k + \gamma \max_{a' \in \mathcal{A}} Q^N_k(x_{k+1}, a') - Q^N_k(\zeta_k) \bigg{)}  \la \partial_c f(c, w )   \sigma (w  \cdot \zeta_k), \nu_k^N \ra   \notag \\
&+ N^{-3/2} \alpha  \bigg{(} r_k + \gamma \max_{a' \in \mathcal{A}} Q^N_k(x_{k+1}, a') - Q^N_k(\zeta_k) \bigg{)}     \la c \sigma' (w  \cdot \zeta_k) \partial_w  f(c,w )\cdot \zeta_{k}, \nu_k^N \ra + O_{p}\left(N^{-2}\right).\label{Eq:nu_dynamicsPrelimit}
\end{align}

Similarly, we can also obtain that
\begin{align}
\la f, \mu^N_t \ra =&  \la f, \mu^N_0 \ra  + \sum_{k=0}^{\floor*{N t}-1} \bigg{(} \la f , \nu^N_{k+1} \ra - \la f , \nu^N_k \ra \bigg{)} \notag \\
=&  \la f, \mu^N_0 \ra  +   N^{-3/2}  \sum_{k=0}^{\floor*{N t}-1}  \alpha  \bigg{(} r_k + \gamma \max_{a' \in \mathcal{A}} Q^N_k(x_{k+1}, a') - Q^N_k(\zeta_k) \bigg{)}  \la \partial_c f(c, w )   \sigma (w  \cdot \zeta_k), \nu_k^N \ra   \notag \\
+& N^{-3/2}  \sum_{k=0}^{\floor*{N t}-1}   \alpha  \bigg{(} r_k + \gamma \max_{a' \in \mathcal{A}} Q^N_k(x_{k+1}, a') - Q^N_k(\zeta_k) \bigg{)}     \la c \sigma' (w  \cdot \zeta_k) \partial_w  f(c,w )\cdot \zeta_{k}, \nu_k^N \ra   + O_{p}\left(N^{-1}\right). \notag \\
\label{muEvolutionWithRemainderTerms}
\end{align}

\subsection{Relative Compactness} \label{RelativeCompactness}

In this section we prove that the family of processes $\{  \mu^N, h^N \}_{N \in \mathbb{N}}$ is relatively compact. Section \ref{CompactContainment} proves compact containment. Section \ref{Regularity} proves regularity. Section \ref{ProofOfRelativeCompactness} combines these results to prove relative compactness.

\subsubsection{Compact Containment} \label{CompactContainment}

We first establish a priori bounds for the parameters $(C_k^i, W_k^i)$.

\begin{lemma} \label{CandWbounds}

For all $i\in\mathbb{N}$ and all $k$ such that $k/N\leq T$,

\begin{align}
| C_k^i | &< C < \infty \notag \\
\mathbb{E} \norm{ W_{k}^{i} } &< C < \infty.\nonumber
\end{align}
\end{lemma}
\begin{proof}

The unimportant finite constant $C<\infty$ may change from line to line. We first observe that
\begin{align*}
| C_{k+1}^i  | \leq&   | C_{k}^i | + \alpha N^{-3/2} \left|  r_k + \gamma \max_{a' \in \mathcal{A}} Q^N_k(x_{k+1}, a') - Q^N_k(x_k, a_k) \right|  | \sigma (W^i_k \cdot  x_k) | \notag \\
\leq&  | C_{k}^i | + \frac{  C  | r_k |  }{N^{3/2}} +  \frac{C}{N^2} \sum_{i=1}^N | C_k^i |,
\end{align*}
where the last inequality follows from the definition of $Q_{k}^N(x,a)$ and the uniform boundedness assumption on $\sigma(\cdot)$.

Then, we subsequently obtain that
\begin{align}
| C_{k}^i | =& | C_{0}^i |  + \sum_{j = 1}^k \bigg{[} | C_{j}^i | - | C_{j-1}^i | \bigg{]} \notag \\
\leq& | C_{0}^i |  + \sum_{j=1}^k  \frac{ C   }{N^{3/2}}  +  \frac{C}{N^2} \sum_{j=1}^k \sum_{i=1}^N | C_{j-1}^i | \notag \\
\leq& | C_{0}^i |  + \frac{C}{\sqrt{N}}  +  \frac{C}{N^2} \sum_{j=1}^k \sum_{i=1}^N | C_{j-1}^i |.
\label{Cbound0011}
\end{align}

This implies that
\begin{align*}
\frac{1}{N} \sum_{i=1}^N | C_{k}^i | \leq& \frac{1}{N} \sum_{i=1}^N | C_{0}^i |  +  \frac{C}{\sqrt{N}} +  \frac{C}{N^2} \sum_{j=1}^k \sum_{i=1}^N | C_{j-1}^i |,
\end{align*}

Let us now define $m_{k}^{N}=\frac{1}{N} \displaystyle \sum_{i=1}^N | C_{k}^i |$. Since the random variables $C_0^i$ take values in a compact set, we have that $ \frac{1}{N} \displaystyle \sum_{i=1}^N | C_{0}^i |  +  \frac{C}{\sqrt{N}}  < C < \infty$. Then,
\begin{align*}
m_{k}^{N} \leq& C +  \frac{C}{N} \sum_{j=1}^k m_{j-1}^{N}.
\end{align*}

By the discrete Gronwall lemma and using $k/N\leq T$,
\begin{equation}
m_k^N \leq  C \exp \bigg{(} \frac{C k}{N} \bigg{)} \leq C.
\label{mBound}
\end{equation}

Note that the constants may depend on $T$.  We can now combine the bounds (\ref{mBound}) and (\ref{Cbound0011}) to yield, for any $0 \leq k \leq \floor*{TN}$,
\begin{align}
| C_{k}^i | \leq& | C_{0}^i |  + \frac{C}{\sqrt{N}}  +  \frac{C}{N^2} \sum_{j=1}^k  m_{j-1}^N \notag \\
\leq& | C_{0}^i |  + \frac{C}{\sqrt{N}}  +  \frac{C}{N} \notag \\
\leq& C,
\label{Cbound0022}
\end{align}
where the last inequality follows from the random variables $C_0^i$ taking values in a compact set.

Now, we turn to the bound for $\parallel W^i_k \parallel$. We start with the bound (using Young's inequality)
\begin{align}
 \parallel W^i_{k+1} \parallel &\leq  \parallel W^i_k \parallel +  \frac{C}{N^{3/2}} \left| r_k + \gamma \max_{a' \in \mathcal{A}} Q^N_k(x_{k+1}, a') - Q^N_k(x_k, a_k) \right| | C^i_k|    | \sigma'(W^i_k \cdot x_k ) | \parallel x_k \parallel\nonumber\\
 &\leq \parallel W^i_k \parallel +   C \left(\frac{1}{N^{3/2}}  + \frac{1}{N^{2}}\sum_{j=1}^N | C^j_k |^{2} \right)\nonumber\\
 &\leq \parallel W^i_k \parallel +   \frac{C}{N},\nonumber
\end{align}
for a constant $C<\infty$ that may change from line to line. Taking an expectation,  using Assumption \ref{A:Assumption2}, the bound (\ref{Cbound0022}), and using the fact that $k/N\leq T$, we obtain
\begin{align*}
\mathbb{E}\parallel W^i_k \parallel\leq C<\infty, %\label{Eq:Bound_w}
\end{align*}
for all $i\in\mathbb{N}$ and all $k$ such that $k/N\leq T$, concluding the proof of the lemma.

\end{proof}

Using the bounds from Lemma \ref{CandWbounds}, we can now establish a bound for $Q_k^N(x,a)$ for $(x,a) \in \mathcal{X} \times \mathcal{A}$.
\begin{lemma} \label{GLemmaBound}
For all $i\in\mathbb{N}$, all $k$ such that $k/N\leq T$, %and any $(x,a) \in \mathcal{X} \times \mathcal{A}$,
\begin{align}
\mathbb{E} \sup_{(x,a)\in\mathcal{X}\times\mathcal{A} }\bigg{[} | Q_k^N(x,a) |^2 \bigg{]} &< C < \infty.\nonumber
\end{align}
\end{lemma}
\begin{proof}

Recall equation (\ref{gEvolution1b}), which describes the evolution of $Q_k^N(x,a)$. Recall the notation $\zeta = (x,a)$ and $\zeta_k = (x_k, a_k)$.
\begin{align}
Q_{k+1}^N(\zeta)  =& Q_{k}^N(\zeta) + \frac{\alpha}{N^2}  \bigg{(} r_k + \gamma \max_{a' \in \mathcal{A}} Q^N_k(x_{k+1}, a') - Q^N_k(\zeta_k) \bigg{)} \sum_{i=1}^N \sigma(W^i_k \cdot \zeta_k ) \sigma(W^i_k \cdot \zeta)   \notag \\
+&  \frac{\alpha}{N^2}   \bigg{(} r_k + \gamma \max_{a' \in \mathcal{A}} Q^N_k(x_{k+1}, a') - Q^N_k(\zeta_k) \bigg{)}  \sum_{i=1}^N \sigma'(W^i_k \cdot \zeta)    \sigma'(W^i_k \cdot \zeta_k ) \zeta_k\cdot \zeta  (C^{i}_k)^2 + \frac{\tilde{C}(\omega)}{N^{3/2}}, \notag
%\label{gEvolution1b2}
\end{align}
where $\tilde{C}(\omega)$ is a random variable (independent of $N$) that is bounded in mean square sense. This leads to the bound
\begin{align}
\sup_{\zeta\in\mathcal{X}\times\mathcal{A}} | Q_{k+1}^N(\zeta) |  \leq& \sup_{\zeta\in\mathcal{X}\times\mathcal{A}} | Q_{k}^N(\zeta) |   + \frac{C}{N} \sup_{\zeta \in\mathcal{X}\times\mathcal{A}} | Q_k^N(\zeta) |  + \frac{\tilde{C}(\omega)}{N}.\nonumber
\end{align}

We now square both sides of the above inequality.
\begin{align}
\sup_{\zeta\in\mathcal{X}\times\mathcal{A}} | Q_{k+1}^N(\zeta) |^2  \leq& \big{(} \sup_{\zeta\in\mathcal{X}\times\mathcal{A}} | Q_{k}^N(\zeta) | + \frac{C}{N} \sup_{\zeta\in\mathcal{X}\times\mathcal{A}} | Q_k^N(\zeta) |  + \frac{C}{N} \big{)}^2 \notag \\
\leq& \sup_{\zeta\in\mathcal{X}\times\mathcal{A}} | Q_{k}^N(\zeta) |^2 + \frac{C}{N} \sup_{\zeta\in\mathcal{X}\times\mathcal{A}}   | Q_{k}^N(\zeta) |^2  + \frac{\tilde{C}^{2}(\omega)}{N},\nonumber
\end{align}
where the last line uses Young's inequality. Therefore, we obtain
\begin{align}
| \sup_{\zeta\in\mathcal{X}\times\mathcal{A}} Q_{k+1}^N(\zeta) |^2 - \sup_{\zeta\in\mathcal{X}\times\mathcal{A}} | Q_{k}^N(\zeta) |^2 \leq& \frac{C}{N} \sup_{\zeta\in\mathcal{X}\times\mathcal{A}} | Q_k^N(\zeta) |^2  + \frac{\tilde{C}^{2}(\omega)}{N}.\nonumber
\end{align}

Then, using a telescoping series, we have
\begin{align}
\sup_{\zeta\in\mathcal{X}\times\mathcal{A}} | Q_{k}^N(\zeta) |^2 =& \sup_{\zeta\in\mathcal{X}\times\mathcal{A}} | Q_{0}^N(\zeta) |^2 + \sum_{j=1}^k \bigg{(} \sup_{\zeta\in\mathcal{X}\times\mathcal{A}} | Q_{j}^N(\zeta) |^2 -  \sup_{\zeta\in\mathcal{X}\times\mathcal{A}} | Q_{j-1}^N(\zeta) |^2 \bigg{)}  \notag \\
\leq& \sup_{\zeta\in\mathcal{X}\times\mathcal{A}}  | Q_{0}^N(\zeta) |^2 + \sum_{j=1}^k  \bigg{(}  \frac{C}{N} \sup_{\zeta\in\mathcal{X}\times\mathcal{A}} | Q_{j-1}^N(\zeta) |^2 + \frac{\tilde{C}^{2}(\omega)}{N} \bigg{)} \notag \\
\leq&  \sup_{\zeta\in\mathcal{X}\times\mathcal{A}} | Q_{0}^N(\zeta) |^2 +  \frac{C}{N}  \sum_{j=1}^k \sup_{\zeta\in\mathcal{X}\times\mathcal{A}} | Q_{j-1}^N(\zeta) |^2 + \tilde{C}^{2}(\omega).\nonumber
\end{align}

Taking expectations, we subsequently obtain
\begin{align}
\mathbb{E} \bigg{[} \sup_{\zeta\in\mathcal{X}\times\mathcal{A}} | Q_{k}^N(\zeta) |^2  \bigg{]} \leq \mathbb{E} \bigg{[} \sup_{\zeta\in\mathcal{X}\times\mathcal{A}} | Q_{0}^N(\zeta) |^2  \bigg{]}  +  \frac{C}{N}  \sum_{j=1}^k   \mathbb{E} \bigg{[} \sup_{\zeta\in\mathcal{X}\times\mathcal{A}} | Q_{j-1}^N(\zeta) |^2 \bigg{]} + C.
\label{gSumBound}
\end{align}

Recall that
\begin{align}
Q_0^N(\zeta) = \frac{1}{\sqrt{N}} \sum_{i=1}^N C^i_0 \sigma \big{(} W^i_0 \cdot (\zeta) \big{)},\nonumber
\end{align}
where $(C^i_0, W^i_0)$ are i.i.d., mean-zero random variables. Then,
\begin{align}
\mathbb{E} \bigg{[} \sup_{\zeta} | Q_0^N(\zeta)  |^2 \bigg{]}   \leq& \mathbb{E} \bigg{[} \sum_{\zeta \in \mathcal{X} \times \mathcal{A}} | Q_0^N(\zeta)  |^2 \bigg{]} \notag \\
\leq& \sum_{\zeta \in \mathcal{X} \times \mathcal{A}} \mathbb{E} \bigg{[} \bigg{(} \frac{1}{\sqrt{N}} \sum_{i=1}^N C^i_0 \sigma \big{(} W^i_0 \cdot (\zeta) \big{)} \bigg{)}^2 \bigg{]} \notag \\
\leq& \frac{C}{N} \sum_{i=1}^N \mathbb{E} \bigg{[} ( C^i_0 )^2 \bigg{]} \notag \\
\leq& C.\nonumber
\end{align}

Substituting this bound into equation (\ref{gSumBound}) produces the desired bound
\begin{align}
\mathbb{E} \left[ \sup_{(x,a)\in\mathcal{X}\times\mathcal{A}} | Q_{k}^N(x,a) |^2  \right] \leq C,\nonumber
\end{align}
for any $0 \leq k \leq \floor*{NT}$.

\end{proof}

We now prove compact containment for the process $\{ (\mu_t^N, h_t^N), t \in [0,T]\}_{N\in\mathbb{N}}$. Recall that $(\mu_t^N, h_t^N) \in D_E([0,T])$ where $E = \mathcal{M}(\mathbb{R}^{1+d}) \times \mathbb{R}^M$ and $M=|\mathcal{X}\times\mathcal{A}|$.

\begin{lemma}\label{L:CompactContainment}
For each $\eta > 0$, there is a compact subset $\mathcal{K}$ of E such that
\begin{align*}
\sup_{N \in \mathbb{N}, 0 \leq t \leq T} \mathbb{P}[ (\mu_t^N, h_t^N) \notin \mathcal{K} ] < \eta.
\end{align*}
\end{lemma}
\begin{proof}
For each $L>0$, define $K_L=[-L,L]^{1+d}$.  Then, we have that $K_L$ is a compact  subset of $\mathbb{R}^{1+d}$, and for each $t\geq 0$ and $N\in \mathbb{N}$,
\begin{equation*}
\mathbb{E}\left[\mu^N_t(\mathbb{R}^{1+d}\setminus K_L)\right] = \frac{1}{N}\sum_{i=1}^N \mathbb{P}\left[ |C^i_{\floor*{N t}}|+\parallel W^i_{\floor*{N t}} \parallel \geq L\right] \leq \frac{C}{L},
\end{equation*}
where we have used Markov's inequality and the bounds from Lemma \ref{CandWbounds}. We define the compact subsets of $\mathcal{M} ( \mathbb{R}^{1+d})$
\begin{equation*}
\hat{K}_L  = \overline{\left\{ \nu:\, \nu(\mathbb{R}^{1+d}\setminus K_{(L+j)^2}) < \frac{1}{\sqrt{L+j}} \textrm{ for all } j\in \mathbb{N}\right\}}
\end{equation*}
and we observe that
\begin{align*}
 \mathbb{P}\left\{ \mu^N_t\not \in \hat{K}_L\right] &\leq \sum_{j=1}^\infty \mathbb{P}\left[ \mu^N_t(\mathbb{R}^{1+d}\setminus K_{(L+j)^2} )> \frac{1}{\sqrt{L+j}}\right]
\leq \sum_{j=1}^\infty \frac{\mathbb{E}[\mu^N_t(\mathbb{R}^{1+d}\setminus K_{(L+j)^2})]}{1/\sqrt{L+j}}\\
&\le \sum_{j=1}^\infty \frac{C}{(L+j)^2/\sqrt{L+j}}
\le \sum_{j=1}^\infty \frac{C}{(L+j)^{3/2}}.
\end{align*}
Given that $\displaystyle \lim_{L\to \infty}\sum_{j=1}^\infty\frac{C}{(L+j)^{3/2}} =0$, we have that, for each $\eta > 0$, there exists a compact set $\hat{K}_L$ such that

\begin{align}
\sup_{N \in \mathbb{N}, 0 \leq t \leq T} \mathbb{P}[ \mu_t^N \notin \hat{K}_L ] < \frac{\eta}{2}.\nonumber
\end{align}

Due to Lemma \ref{GLemmaBound} and Markov's inequality, we also know that, for each $\eta > 0$, there exists a compact set $U = [-B, B]^{M}$ such that
\begin{align}
\sup_{N \in \mathbb{N}, 0 \leq t \leq T} \mathbb{P}[  h_t^N \notin U]  < \frac{\eta}{2}.\nonumber
\end{align}

Therefore, for each $\eta > 0$, there exists a compact set $ \hat K_L \times [-B, B]^{M} \subset E$ such that
\begin{align}
\sup_{N \in \mathbb{N}, 0 \leq t \leq T} \mathbb{P}[  (\mu_t^N, h_t^N) \notin \hat K_L \times [-B, B]^{M}]  < \eta.\nonumber
\end{align}

\end{proof}

\subsubsection{Regularity} \label{Regularity}

We now establish regularity of the process $\mu^{N}$ in $D_{\mathcal{M}(\mathbb{R}^{1+d})}([0,T])$. Define the function $q(z_{1},z_{2})=\min\{|z_{1}-z_{2}|,1\}$ where $z_{1},z_{2} \in \mathbb{R}$. Let $\mathcal{F}^N_t$ be the $\sigma-$algebra generated by $\{(C_{0}^{1},W_{0}^{i})\}_{i=1}^{N}$ and $\{x_{j}\}_{j=0}^{\floor{NT}-1}$.

\begin{lemma}\label{MU:regularity}
Let  $f \in C^{2}_{b}(\mathbb{R}^{1+d})$. For any $\delta \in (0,1)$, there is a constant $C<\infty$ such that for $0\leq u\leq \delta$,  $0\leq v\leq \delta\wedge t$, $t\in[0,T]$,
\begin{equation*}
 \mathbb{E}\left[q(\left< f,\mu^N_{t+u}\right>,\left< f,\mu^N_t\right>)q(\left< f,\mu^N_t\right>,\left< f,\mu^N_{t-v}\right>)\big| \mathcal{F}^N_t\right]  \le  C \delta + \frac{C}{N^{3/2}}.
\end{equation*}
\end{lemma}
\begin{proof}
We start by noticing that a Taylor expansion gives for $0\leq s\leq t\leq T$ \begin{align}
| \la f , \mu^N_{t} \ra -  \la f , \mu^N_{s} \ra  | =& | \la f,   \nu^N_{\floor*{N t} } \ra -  \la f,   \nu^N_{\floor*{N s} } \ra | \notag \\
\leq& \frac{1}{N} \sum_{i=1}^N | f( C^i_{\floor*{N t}}, W^i_{\floor*{N t}}) - f(C^i_{\floor*{N s}}, W^i_{\floor*{N s}}) | \notag \\
\leq& \frac{1}{N} \sum_{i=1}^N  | \partial_c f (  \bar C^i_{\floor*{N t}}, \bar W^i_{\floor*{N t}} ) | |  C^i_{\floor*{N t}} -  C^i_{\floor*{N s}} | \notag \\
+& \frac{1}{N} \sum_{i=1}^N  \parallel \partial_w f (  \bar C^i_{\floor*{N t}}, \bar W^i_{\floor*{N t}} ) \parallel \parallel  W^i_{\floor*{N t}} -  W^i_{\floor*{N s}} \parallel,
\label{Regularity1}
\end{align}
for points $\bar C^{i}, \bar W^{i}$ in the segments connecting $C^i_{\floor*{N s}}$ with $C^i_{\floor*{N t}}$ and  $W^i_{\floor*{N s}}$ with $W^i_{\floor*{N t}}$, respectively.

Let's now establish a bound on $|  C^i_{\floor*{N t}} -  C^i_{\floor*{N s}} |$ for $s < t \leq T$ with $0<t-s\leq \delta<1$.
\begin{align}
 \mathbb{E} \bigg{[} |  C^i_{\floor*{N t}} -  C^i_{\floor*{N s}} | \bigg| \mathcal{F}^N_s   \bigg{]} =& \mathbb{E} \bigg{[} | \sum_{k = \floor*{N s } }^{\floor*{N t}-1} ( C_{k+1}^i - C_k^i  ) |  \bigg| \mathcal{F}^N_s   \bigg{]}  \notag \\
\leq&  \mathbb{E} \bigg{[} \sum_{k = \floor*{N s } }^{\floor*{N t}-1}  | \alpha \big{(} r_k + \gamma \max_{a' \in \mathcal{A}} Q^N_k(x_{k+1}, a') - Q^N_k(x_k, a_k)  \big{)} \frac{1}{N^{3/2}} \sigma (W^i_k \cdot x_k) |  \bigg| \mathcal{F}^N_s   \bigg{]}  \notag \\
\leq& \frac{1}{N^{3/2}} \sum_{k = \floor*{N s } }^{\floor*{N t}-1} C  \leq \frac{C}{\sqrt{N}} (t -s ) + \frac{C}{N^{3/2}} \notag \\
%\leq& \frac{C}{\sqrt{N}} (t - s)^{p} \mathbf{1}_{ t -s < 1 }  + \frac{C}{\sqrt{N}} (t -s )^{p} T^{1/p} \mathbf{1}_{t -s \geq 1}  + \frac{C}{N^{3/2}} \notag \\
\leq& \frac{C}{\sqrt{N}} \delta + \frac{C}{N^{3/2}},
\label{CregularityBound00}
\end{align}
where Assumption \ref{A:Assumption2} was used as well as the bounds from Lemmas \ref{CandWbounds} and \ref{GLemmaBound}.

Let's now establish a bound on $\parallel  W^i_{\floor*{N t}} -  W^i_{\floor*{N s}} \parallel$ for $s < t \leq T$ with $0<t-s\leq \delta<1$. We obtain
\begin{align}
 & \mathbb{E} \bigg{[} \parallel W^i_{\floor*{N t}} -  W^i_{\floor*{N s}} \parallel \bigg| \mathcal{F}^N_s   \bigg{]}  = \mathbb{E} \bigg{[} \parallel \sum_{k = \floor*{N s } }^{\floor*{N t}-1} ( W_{k+1}^i - W_k^i  ) \parallel \bigg| \mathcal{F}^N_s   \bigg{]}   \notag \\
& \leq   \mathbb{E}  \bigg{[} \sum_{k = \floor*{N s } }^{\floor*{N t}-1} \parallel    \alpha \big{(} r_k + \gamma \max_{a' \in \mathcal{A}} Q^N_k(x_{k+1}, a') - Q^N_k(x_k, a_k)  \big{)}  \frac{1}{N^{3/2}} C^i_k \sigma' (W^i_k \cdot x_k) x_k \parallel \bigg| \mathcal{F}^N_s   \bigg{]}   \notag \\
& \leq\frac{1}{N^{3/2}} \sum_{k = \floor*{N s } }^{\floor*{N t}-1} C  \notag \\
& \leq  \frac{C}{\sqrt{N}} (t -s ) + \frac{C}{N} \leq \frac{C}{\sqrt{N}} \delta + \frac{C}{N^{3/2}},
\label{WregularityBound00}
\end{align}
where we have again used the bounds from Lemmas \ref{CandWbounds} and \ref{GLemmaBound}.

Now, we return to equation (\ref{Regularity1}). Due to Lemma \ref{CandWbounds}, the quantities $( \bar C^i_{\floor*{N t}}, \bar W^i_{\floor*{N t}} )$ are bounded in expectation for $0 < s < t \leq T$.  Therefore, for $0 < s < t \leq T$ with $0<t-s\leq \delta<1$
\begin{align*}
\mathbb{E}\left[| \la f , \mu^N_{t} \ra -  \la f , \mu^N_{s} \ra  | \big| \mathcal{F}^N_s \right] \leq C \delta + \frac{C}{N^{3/2}}.
\end{align*}
where $C<\infty$ is some unimportant constant. Then, the statement of the Lemma follows.
\end{proof}

We next establish regularity of the process $h^{N}_t$ in $D_{\mathbb{R}^M}([0,T])$. For the purposes of the following lemma, let the function $q(z_{1},z_{2})=\min\{ \norm{z_{1}-z_{2}},1\}$ where $z_{1},z_{2} \in \mathbb{R}^M$ and $\norm{z} = |z_1 | + \cdots +  |z_M|$.  \\

\begin{lemma}\label{H:regularity}
For any $\delta \in (0,1)$, there is a constant $C<\infty$ such that for $0\leq u\leq \delta<1$,  $0\leq v\leq \delta\wedge t$, $t\in[0,T]$,

\begin{equation*}
 \mathbb{E}\left[q(h_{t+u}^N, h_t^N )q(h_t^N,h_{t-v}^N )\big| \mathcal{F}^N_t \right]  \le  C \delta + \frac{C}{N}.
\end{equation*}

\end{lemma}
\begin{proof}

Recall that
\begin{align}
Q_{k+1}^N(\zeta)  = Q_{k}^N(\zeta)  + \frac{1}{\sqrt{N}} \sum_{i=1}^N  \bigg{(} C^i_{k+1} - C^i_k )  \sigma \big{(} W^i_{k+1} \cdot \zeta \big{)}   +  \sigma' \big{(} W^{i,\ast}_k \cdot \zeta \big{)} \zeta\cdot ( W^i_{k+1} - W^i_k )   C^i_k \bigg{)}.\nonumber
\end{align}

Therefore,

\begin{align}
h_t^N(\zeta) - h_s^N(\zeta) =& Q_{\floor*{N t}}(\zeta) - Q_{\floor*{N s}}(\zeta) \notag \\
=& \sum_{k= \floor*{N s}}^{\floor*{N t}} ( Q_{k+1}^N(\zeta)  -  Q_{k}^N(\zeta) ) \notag \\
=& \sum_{k= \floor*{N s}}^{\floor*{N t}}  \frac{1}{\sqrt{N}} \sum_{i=1}^N  \bigg{(} C^i_{k+1} - C^i_k )  \sigma(W^i_{k+1} \cdot \zeta)   +  \sigma'(W^{i,\ast}_k \cdot \zeta) \zeta \cdot ( W^i_{k+1} - W^i_k )   C^i_k \bigg{)}.\nonumber
\end{align}

This yields the bound
\begin{align}
| h_t^N(\zeta) - h_s^N(\zeta)  | \leq&  \sum_{k= \floor*{N s}}^{\floor*{N t}} |  Q_{k+1}^N(\zeta)  -  Q_{k}^N(\zeta) | \notag \\
\leq& \sum_{k= \floor*{N s}}^{\floor*{N t}} \frac{1}{\sqrt{N}} \sum_{i=1}^N  \bigg{(} | C^i_{k+1} - C^i_k |    + \norm{ W^i_{k+1} - W^i_k } \bigg{)},\nonumber
\end{align}
where we have used the boundedness of $\sigma'(\cdot)$ (from Assumption \ref{A:Assumption2}) and the bounds from Lemma \ref{CandWbounds}.

Taking expectations,
\begin{align}
\mathbb{E} \bigg{[} \sup_{\zeta}| h_t^N(\zeta) - h_s^N(\zeta)  | \bigg{|} \mathcal{F}_s^N \bigg{]} \leq&   \frac{1}{\sqrt{N}} \sum_{i=1}^N   \sum_{k= \floor*{N s}}^{\floor*{N t}} \mathbb{E} \bigg{[} | C^i_{k+1} - C^i_k |    + \norm{ W^i_{k+1} - W^i_k } \bigg{|} \mathcal{F}_s^N \bigg{]}.\nonumber
\end{align}

Using the bounds (\ref{CregularityBound00}) and (\ref{WregularityBound00}),
\begin{align}
\mathbb{E} \bigg{[} \sup_{\zeta}| h_t^N(\zeta) - h_s^N(\zeta)  | \bigg{|} \mathcal{F}_s^N \bigg{]} \leq&   \frac{1}{\sqrt{N}} \sum_{i=1}^N \bigg{(} \frac{C}{\sqrt{N}}(t-s) + \frac{C}{N^{3/2}} \bigg{)} \notag \\
=&  C (t -s ) + \frac{C}{N}.
\label{GregularityBound00}
\end{align}

Therefore, we have obtained that
\begin{align}
\mathbb{E} \bigg{[} \norm{h_t^N - h_s^N} \bigg{|} \mathcal{F}_s^N \bigg{]} \leq C (t -s ) + \frac{C}{N}.\nonumber
\end{align}

The statement of the Lemma then follows.

\end{proof}

\subsubsection{Combining our results to prove relative compactness} \label{ProofOfRelativeCompactness}

\begin{lemma}\label{L:RelativeCompactness}
The family of processes $\{\mu^N, h^N \}_{N\in\mathbb{N}}$ is relatively compact in $D_{E}([0,T])$.
\end{lemma}

\begin{proof}
Combining Lemmas \ref{L:CompactContainment} and \ref{MU:regularity}, and Theorem 8.6 of Chapter 3 of \cite{EthierAndKurtz} proves that $\{\mu^N\}_{N\in\mathbb{N}}$ is relatively compact in $D_{\mathcal{M}(\mathbb{R}^{1+d})}([0,T])$. (See also Remark 8.7 B of Chapter 3 of \cite{EthierAndKurtz} regarding replacing $\sup_N$ with $\lim_N$ in the regularity condition B of Theorem 8.6.). Similarly, combining Lemmas \ref{L:CompactContainment} and \ref{H:regularity} proves that $\{h^N\}_{N\in\mathbb{N}}$ is relatively compact in $D_{\mathbb{R}^{M}}([0,T])$.

From these, we finally obtain that  $\{\mu^N, h^N \}_{N\in\mathbb{N}}$ is relatively compact as a $D_E([0,T])-$valued random variable where $E = \mathcal{M}(\mathbb{R}^{1+d}) \times \mathbb{R}^M$.

\end{proof}

\subsection{Identification of the Limit} \label{Identification}

We must first establish that $M_t^{1,N}, M_t^{2,N}, M_t^{3,N} \overset{p} \rightarrow 0$ as $N \rightarrow \infty$. For this purpose, we first prove two lemmas.
\begin{lemma} \label{IdentLemma1}
Consider a Markov chain $z_k$ on a finite, discrete space $\mathcal{S}$ with a unique limiting distribution $q(z)$ and a random function $f^N: \mathcal{S} \rightarrow \mathbb{R}$. Suppose $f^N$ is uniformly bounded in $L^2$ with respect to N. Then,
\begin{align}
\lim_{N \rightarrow \infty}\mathbb{E} \bigg{[} \bigg{|} \frac{1}{N} \sum_{k=0}^{N-1} f^{N} ( z_k ) - \sum_{z\in\mathcal{S}} f^{N}(z) q(z) \bigg{|} \bigg{]}  = 0.\notag
\end{align}

\end{lemma}

\begin{proof}[Proof of Lemma \ref{IdentLemma1}]
The proof of Lemma \ref{IdentLemma1} should be known. However, given that we could not locate an exact reference we provide its short proof here. We begin by recognizing that
\begin{align}
\frac{1}{N} \sum_{k=0}^{N-1} f^{N} ( z_k ) - \sum_{z\in\mathcal{S}} f^{N}(z) q(z)  =& \sum_{s \in \mathcal{S}} \bigg{(} \frac{f^{N} ( s )}{N} \sum_{k=0}^{N-1}    \mathbf{1}_{z_k = s } -   f^{N}(s) q(s)  \bigg{)}  \notag \\
=& \sum_{s \in \mathcal{S}} f^{N} ( s ) \bigg{(} \frac{1}{N} \sum_{k=0}^{N-1}    \mathbf{1}_{z_k = s }  -   q(s) \bigg{)}.\notag
%\label{ergodic1}
\end{align}

Of course, we have that $\displaystyle \frac{1}{N} \sum_{k=0}^{N-1}   \mathbf{1}_{z_k = s }  \overset{p} \rightarrow q(s)$ and, since $\displaystyle \frac{1}{N} \sum_{k=0}^{N-1}   \mathbf{1}_{z_k = s }$ is uniformly bounded, a special case of Vitali's theorem gives
\begin{align}
\lim_{N \rightarrow \infty} \mathbb{E} \bigg{[} \bigg{(} \frac{1}{N} \sum_{k=0}^{N-1}    \mathbf{1}_{z_k = s }  -   q(s) \bigg{)}^2 \bigg{]} = 0.\notag
\end{align}

Using the Cauchy-Schwartz inequality, we have
\begin{align}
\mathbb{E} \bigg{[} \bigg{|} \frac{1}{N} \sum_{k=0}^{N-1} f^{N} ( z_k ) - \sum_{z\in\mathcal{S}} f^{N}(z) q(z) \bigg{|} \bigg{]} \leq& \sum_{s \in \mathcal{S}}  \mathbb{E} \bigg{[}  | f^{N} ( s ) | |\frac{1}{N} \sum_{k=0}^{N-1}    \mathbf{1}_{z_k = s }  -   q(s) | \bigg{]} \notag \\
\leq& \sum_{s \in \mathcal{S}}  \mathbb{E} \bigg{[}  \big{(} f^{N} ( s ) \big{)}^2 \bigg{]}^{\frac{1}{2}}  \mathbb{E} \bigg{[} \bigg{(} \frac{1}{N} \sum_{k=0}^{N-1}    \mathbf{1}_{z_k = s }  -   q(s) \bigg{)}^2 \bigg{]}^{\frac{1}{2}} \notag \\
\leq& C \sum_{s \in \mathcal{S}}  \mathbb{E} \bigg{[} \bigg{(} \frac{1}{N} \sum_{k=0}^{N-1}    \mathbf{1}_{z_k = s }  -   q(s) \bigg{)}^2 \bigg{]}^{\frac{1}{2}} \notag
\end{align}

Therefore, we obtain that
\begin{align}
\lim_{N \rightarrow \infty}\mathbb{E} \bigg{[} \bigg{|} \frac{1}{N} \sum_{k=0}^{N-1} f^{N} ( z_k ) - \sum_{z\in\mathcal{S}} f^{N}(z) q(z) \bigg{|} \bigg{]}  = 0,\notag
\end{align}
concluding the proof of the lemma.
\end{proof}

\begin{lemma} \label{IdentLemma2}
Consider the notation and assumptions of Lemma \ref{IdentLemma1}. Define the quantity
\begin{align}
M_t^N =& \frac{1}{N} \sum_{k=0}^{\floor*{N t}-1}    f_k^N(z_k ) -   \frac{1}{N} \sum_{k=0}^{\floor*{N t}-1}  \sum_{z\in\mathcal{S}}   f_k^N(z ) q(z),\nonumber
\end{align}
where the function $f_k^N$ satisfies
\begin{align}
\sup_{z \in \mathcal{S}} \mathbb{E} \bigg{[} | f_k^N(z) - f_{k-1}^N(z) | \bigg{]} \leq& \frac{C}{N}, \notag \\
\sup_{0 \leq k \leq \floor{TN}} \sup_{z \in\mathcal{S}}\mathbb{E} \bigg{[} | f_k^N(z) |^2 \bigg{]} <& C. \label{LemmaAssumptions2}
%\label{BoundsF}
\end{align}
Then we have that,
\begin{align}
\phantom{.}& \lim_{N \rightarrow \infty} \sup_{t\in(0,T]}\mathbb{E} | M_t^N | = 0.\nonumber
\end{align}

\end{lemma}
\begin{proof}

For any $K \in \mathbb{N}$ and $\Delta = \frac{t}{K}$, we have
\begin{align}
M_t^N =& \sum_{j=0}^{K-1} \Delta \frac{1}{ \floor*{\Delta N}} \sum_{k = j \floor*{\Delta N}}^{(j+1) \floor*{\Delta N}-1} \bigg{(} f_k^N(z_k) - \sum_{z\in\mathcal{S}} f_k^N(z) q(z)  \bigg{)} + o(1)  \notag \\
=& \sum_{j=0}^{K-1} \Delta \frac{1}{\floor*{\Delta N} } \sum_{k = j \floor*{\Delta N}}^{(j+1) \Delta N-1} \bigg{(} f_{j \floor*{\Delta N}}^N(z_k) - \sum_{z\in\mathcal{S}} f_{j \floor*{\Delta N}}^N(z) q(z)  \bigg{)} \notag \\
& + \sum_{j=0}^{K-1} \Delta \frac{1}{\floor*{\Delta N}} \sum_{k = j \floor*{\Delta N}}^{(j+1) \floor*{\Delta N}-1} \bigg{[} \bigg{(} f_k^N(z_k) - \sum_{z\in\mathcal{S}} f_k^N(z)  q(z) \bigg{)} - \bigg{(}  f_{j \floor*{\Delta N}}^N(z_k) - \sum_{z\in\mathcal{S}} f_{j \floor*{\Delta N}}^N(z) q(z) \bigg{)}  \bigg{]} \notag \\
& + o(1),
\label{MT}
\end{align}
where the term $o(1)$ goes to zero, at least, in $L^1$ as $N\rightarrow\infty$.
We will need to show that, for each $j = 0, 1, \ldots, K-1$,
\begin{align}
 \frac{1}{\floor*{\Delta N} } \sum_{k = j \floor*{\Delta N}}^{(j+1) \floor*{\Delta N}-1}  \mathbf{1}_{z_k = s } \overset{p} \rightarrow q(s) \phantom{....} \textrm{as} \phantom{....} N \rightarrow \infty.
 \label{Induction}
\end{align}

This can be proven in the following way. We already know that $\displaystyle \frac{1}{\floor*{\Delta N} } \sum_{k = 0}^{(j+1) \floor*{\Delta N}-1}  \mathbf{1}_{z_k = s } \overset{p} \rightarrow (j+1) q(s)$ as $N \rightarrow \infty$. Of course, we also have that $\displaystyle \frac{1}{\floor*{\Delta N} } \sum_{k = 0}^{j \floor*{\Delta N}-1}  \mathbf{1}_{z_k = s } \overset{p} \rightarrow j q(s)$ as $N \rightarrow \infty$. Then, it must hold that $\displaystyle \frac{1}{\floor*{\Delta N} } \sum_{k = j \floor*{\Delta N} }^{ (j+1) \floor*{\Delta N}-1}  \mathbf{1}_{z_k = s } \overset{p} \rightarrow q(s)$.

Combining (\ref{Induction}) and Lemma \ref{IdentLemma1}, we can show that for each $j = 0, 1, \ldots, K-1$,
\begin{align}
\lim_{N \rightarrow \infty} \mathbb{E} \bigg{[} \bigg{|} \frac{1}{\floor*{\Delta N} } \sum_{k = j \floor*{\Delta N}}^{(j+1) \floor*{\Delta N}-1} \bigg{(} f_{j \floor*{\Delta N}}^N(z_k) - \sum_{z\in\mathcal{S}} f_{j \floor*{\Delta N}}^N(z) q(z) \bigg{)}  \bigg{|} \bigg{]} = 0.\notag
\end{align}

We next consider the second term in (\ref{MT}). To bound this term, we will use the assumption (\ref{LemmaAssumptions2}).
\begin{align}
\phantom{.}& \mathbb{E}  \bigg{|} \sum_{j=0}^{K-1} \Delta \frac{1}{\floor*{\Delta N}} \sum_{k = j \floor*{\Delta N}}^{(j+1) \floor*{\Delta N}-1} \bigg{[} \bigg{(} f_k^N(z_k) - \sum_{z\in\mathcal{S}} f_k^N(z)  q(z) \bigg{)} - \bigg{(}  f_{j \floor*{\Delta N}}^N(z_k) - \sum_{z\in\mathcal{S}} f_{j \floor*{\Delta N}}^N(z) q(z) \bigg{)}  \bigg{]}  \bigg{|} \notag \\
\leq&  \sum_{j=0}^{K-1} \Delta \frac{1}{\floor*{\Delta N}} \sum_{k = j \floor*{\Delta N}}^{(j+1) \floor*{\Delta N}-1} \mathbb{E} \bigg{|} \bigg{(} f_k^N(z_k) - \sum_{z\in\mathcal{S}} f_k^N(z)  q(z) \bigg{)} - \bigg{(}  f_{j \floor*{\Delta N}}^N(z_k) - \sum_{z\in\mathcal{S}} f_{j \floor*{\Delta N}}^N(z) q(z) \bigg{)} \bigg{|}  \notag \\
\leq&  C \sum_{j=0}^{K-1} \Delta \frac{1}{\floor*{\Delta N}} \sum_{k = j \floor*{\Delta N}}^{(j+1) \floor*{\Delta N}-1} \frac{k-j\floor*{\Delta N}}{N}.
\end{align}
Therefore, we can show that

\begin{align}
\phantom{.}& \mathbb{E}  \bigg{|} \sum_{j=0}^{K-1} \Delta \frac{1}{\floor*{\Delta N}} \sum_{k = j \floor*{\Delta N}}^{(j+1) \floor*{\Delta N}-1} \bigg{[} \bigg{(} f_k^N(z_k) - \sum_{z\in\mathcal{S}} f_k^N(z)  q(z) \bigg{)} - \bigg{(}  f_{j \floor*{\Delta N}}^N(z_k) - \sum_{z\in\mathcal{S}} f_{j \floor*{\Delta N}}^N(z) q(z) \bigg{)}  \bigg{]}  \bigg{|} \notag \\
=&  C \sum_{j=0}^{K-1} \Delta \frac{1}{\floor*{\Delta N}} \sum_{k = 0}^{ \floor*{\Delta N}-1} \frac{k}{N} \notag \\
\leq& C \sum_{j=0}^{K-1} \Delta \frac{1}{\floor*{\Delta N}}  \frac{\floor*{\Delta N}^2 }{N} \notag \\
=& C \sum_{j=0}^{K-1} \Delta  \frac{\floor*{\Delta N}}{N} \notag \\
\leq& C \sum_{j=0}^{K-1} \Delta^2 \notag \\
\leq& C \Delta.
\end{align}

Collecting our results, we have shown that
\begin{align*}
\lim \sup_{N \rightarrow \infty }  \sup_{t\in(0,T]}\mathbb{E} | M_t^N | \leq C \frac{T}{K}.
\end{align*}

Note that $K$  was arbitrary. Consequently, we obtain
\begin{align*}
 \lim_{N \rightarrow \infty} \sup_{t\in(0,T]}\mathbb{E} | M_t^N | &= 0, %\notag \\
%&\phantom{.}& M_t^N \overset{p} \rightarrow 0 \phantom{....} \textrm{as} \phantom{....} N \rightarrow \infty.
\end{align*}
concluding the proof of the lemma.

\end{proof}

This now allows us to prove the following lemma.\\

\begin{lemma} \label{Mlemma}
Recall the definition of $M_t^{1,N}, M_t^{2,N}, M_t^{3,N}$ from (\ref{Eq:h_equationPrelimit2}). Then, we have that $M_t^{1,N}, M_t^{2,N}, M_t^{3,N}\overset{L^1} \rightarrow 0$ as $N \rightarrow \infty$.
\end{lemma}
\begin{proof}
The process $Q_k^N(x,a)$ satisfies the uniform $L^2$ bound in equation (\ref{LemmaAssumptions2}) due to Lemma \ref{GLemmaBound}.  It also satisfies the regularity condition in equation (\ref{LemmaAssumptions2}). Indeed, recalling the notation $\zeta=(x,a)$ and $\zeta_k=(x_k,a_k)$, we have
\begin{align}
 \mathbb{E} \bigg{[} |  Q_{k+1}^N(\zeta)  &-  Q_{k}^N(\zeta) | \bigg{]} \leq \frac{1}{\sqrt{N}} \sum_{i=1}^N  \mathbb{E} \bigg{[} \bigg{|}  (C^i_{k+1} - C^i_k )  \sigma(W^i_{k+1} \cdot \zeta)   +  \sigma'(W^{i,\ast}_k \cdot \zeta) \zeta \cdot ( W^i_{k+1} - W^i_k )   C^i_k \bigg{|} \bigg{]} \notag \\
 &\leq  \frac{1}{N^2} \sum_{i=1}^N  \mathbb{E} \bigg{[}  | \alpha \big{(} r_k + \gamma \max_{a' \in \mathcal{A}} Q^N_k(x_{k+1}, a') - Q^N_k(\zeta_k)  \big{)}  \sigma (W^i_k \cdot \zeta_k) \sigma (W^i_k \cdot s)| \bigg{]} \notag \\
 &+  \frac{1}{N^2} \sum_{i=1}^N  \mathbb{E} \bigg{[}  | \alpha \big{(} r_k + \gamma \max_{a' \in \mathcal{A}} Q^N_k(x_{k+1}, a') - Q^N_k(\zeta_k)  \big{)}  \sigma'(W^{i,\ast}_k \cdot \zeta)   \sigma'(W^{i,\ast}_k \cdot \zeta_k)  \zeta \cdot  \zeta_k  (C^i_k)^2 \bigg{|} \bigg{]}, \notag \\
 &\leq \frac{1}{N^2} \sum_{i=1}^N C = \frac{C}{N},
 \label{DifferenceQBound1}
\end{align}
where we have used the bounds from Lemmas \ref{CandWbounds} and \ref{GLemmaBound}, the boundedness of $\sigma(\cdot)$ and $\sigma'(\cdot)$, and the Cauchy-Schwartz inequality.

In addition,
\begin{eqnarray}
 \mathbb{E} \bigg{[} |  \max_{a \in \mathcal{A}} Q_{k+1}^N(x,a)  - \max_{a \in \mathcal{A}}  Q_{k}^N(x,a) | \bigg{]} &\leq&  \mathbb{E} \bigg{[}   \max_{a \in \mathcal{A}} | Q_{k+1}^N(x,a)  - Q_{k}^N(x,a) | \bigg{]}  \notag \\
 &\leq& \sum_{a \in \mathcal{A}} \mathbb{E} \bigg{[}   | Q_{k+1}^N(x,a)  - Q_{k}^N(x,a) | \bigg{]} \notag \\
 &\leq& \frac{C}{N},
\end{eqnarray}
where we have used the bound (\ref{DifferenceQBound1}).

The term $B_{x,a,x', a',k}^N$ that appears in the formula for $M_{t}^{1,N}$ can be treated analogously using (\ref{Eq:nu_dynamicsPrelimit}) and Lemma \ref{GLemmaBound}. The result for $M_{t}^{1,N}$ then immediately follows from Lemmas \ref{IdentLemma1} and \ref{IdentLemma2} and the triangle inequality. Using the same approach, one can obtain the claim for $M_{t}^{2,N}$ and $M_{t}^{3,N}$, and the proof is omitted due to the similarity of the argument.
\end{proof}

Let $\rho^N$ be the probability measure of $\left(\mu^N, h^N \right)_{0\leq t\leq T}$. Each $\rho^N$ takes values in the set of probability measures $\mathcal{M} \big{(} D_E([0,T]) \big{)}$.
In Section \ref{RelativeCompactness} we have proven that the sequence of measures $\{\rho^{N}\}_{N\in\mathbb{N}}$ is relatively compact. By standard results, see for example Theorem 3.1 in Chapter of \cite{EthierAndKurtz}, this implies that there is a convergent subsequence  $\rho^{N_k}$ which converges weakly. We must prove that any limit point $\rho$ of a convergent subsequence $\rho^{N_k}$ will satisfy the evolution equation (\ref{EvolutionEquationIntroductionXavierInfiniteTime}).

\begin{lemma} \label{IdentificationInfiniteTime}
Let $\rho^{N_k}$ be a convergent subsequence with a limit point $\rho$. Then, $\rho$ is a Dirac measure concentrated on $(\mu, h) \in D_E([0,T])$ and $(\mu, h)$ satisfies equation (\ref{EvolutionEquationIntroductionXavierInfiniteTime}).
\end{lemma}
\begin{proof}
We define a map $F(\mu,h): D_{E}([0,T]) \rightarrow \mathbb{R}_{+}$ for each $t \in [0,T]$, $f \in C^{2}_{b}(\mathbb{R}^{1+d})$, $g_{1},\cdots,g_{p}\in C_{b}(\mathbb{R}^{1+d})$, $q_{1},\cdots,q_{p}\in C_{b}( \mathcal{X} \times \mathcal{A} )$, and $0\leq s_{1}<\cdots< s_{p}\leq t$.
\begin{align}
F(\mu, h) =&  \bigg{|} \left(\la f, \mu_t \ra - \la f,  \mu_0 \ra \right) \times\la g_{1},\mu_{s_{1}}\ra\times\cdots\times \la g_{p},\mu_{s_{p}}\ra\bigg{|} \notag \\
+& \sum_{(x,a) \in \mathcal{X} \times \mathcal{A} }  \bigg{|}  h_t(x,a)  - h_0(x,a) -   \alpha \int_0^t \sum_{(x', a') \in \mathcal{X}  \times \mathcal{A}}    \bigg{(} r(x', a') + \gamma \int_{\mathcal{X}} \max_{a'' \in \mathcal{A}} h_s^N(x'', a'') p(x'' | x',a') - h_s^N(x', a') \bigg{)} \notag \\
&\times \bigg{(} \la \sigma \big{(} w \cdot  (x', a') \big{)} \sigma \big{(} w \cdot (x,a) \big{)}, \mu_s^N \ra + \la c^2 \sigma'(w x)    \sigma' \big{(}w  \cdot (x', a') \big{)} (x', a') \cdot (x,a) , \mu_s^N \ra \bigg{)} \pi(x', a') ds  \bigg{|} \notag \\
&\times q_{1}(h_{s_{1}})  \times\cdots\times q_{p}(h_{s_{p}}).\notag
%\label{EvolutionEquation2}
\end{align}

Then, using equations (\ref{Eq:h_equationPrelimit}) and (\ref{muEvolutionWithRemainderTerms}), we obtain
\begin{align}
\mathbb{E}_{\rho^N} [ F(\mu, h ) ] =& \mathbb{E} [ F( \mu^N, h^N) ] \notag \\
=& \mathbb{E} \left|  \mathcal{O}_p(N^{-1/2})  \times \prod_{i=1}^{p} \la g_{i},\mu^{N}_{s_{i}}\ra \right|   \notag \\
+&  \mathbb{E} \left| (  M_t^{1,N} + M_t^{2,N} + M_t^{3,N} + \mathcal{O}_p(N^{-1/2}) ) \times \prod_{i=1}^{p} q_{i}( h^{N}_{s_{i}}) \right|   \notag \\
\leq& C \bigg{(} \mathbb{E} \bigg{[} | M^{1,N}(t) |  \bigg{]} +  \mathbb{E} \bigg{[}| M^{2,N}(t)  |  \bigg{]}+  \mathbb{E} \bigg{[}| M^{3,N}(t)  | \bigg{]}  \bigg{)}+ O(N^{-1/2}).\notag
\end{align}

Therefore, using Lemma \ref{Mlemma},
\begin{align}
\lim_{N \rightarrow \infty} \mathbb{E}_{\rho^N} [ F(\mu, h) ] = 0.\notag
\end{align}
Since $F(\cdot)$ is continuous and $F( \mu^N)$ is uniformly bounded (due to the uniform boundedness results of Section \ref{RelativeCompactness}),
\begin{align}
 \mathbb{E}_{\rho} [ F(\mu, h) ] = 0.\notag
\end{align}
Since this holds for each $t \in [0,T]$, $f \in C^{2}_{b}(\mathbb{R}^{1+d})$ and $g_{1},\cdots,g_{p}, q_1, \cdots, q_p \in C_{b}(\mathbb{R}^{1+d})$, $(\mu, h)$ satisfies the evolution equation (\ref{EvolutionEquationIntroductionXavierInfiniteTime}).
\end{proof}

\subsection{Uniqueness} \label{Uniqueness}

We prove uniqueness of the limit equation (\ref{EvolutionEquationIntroductionXavierInfiniteTime}) for $h_t$. Suppose there are two solutions $h_t^1$ and $h_t^2$. Let us define their difference to be $\phi_t = h_t^1 - h_t^2$.

Recall that $A$ is the tensor
\begin{align}
A_{x, a, x', a'} = \alpha \la \sigma(w \cdot (x', a') ) \sigma(w \cdot (x,a) ), \mu_0 \ra + \alpha\la c^2 \sigma'(w \cdot (x', a') ) \sigma'(w \cdot (x,a) ) (x',a')\cdot (x,a), \mu_0 \ra.\nonumber
\end{align}

For notational convenience, define $\zeta=(x,a)$, $\zeta'=(x',a')$, and
\begin{align*}
A_{\zeta, \zeta'} = \alpha  \la \sigma \big{(} w \cdot \zeta'  \big{)} \sigma \big{(} w \cdot \zeta \big{)}, \mu_0 \ra +  \alpha  \la c^2 \sigma' \big{(} w \cdot \zeta'  \big{)} \sigma' \big{(} w \cdot \zeta \big{)} \zeta' \cdot \zeta, \mu_0 \ra .%\label{Eq:A_matrixElements}
\end{align*}

The matrix $A$ is positive definite; see Section \ref{ProofOFPositiveDefinite} for the proof. We also define
\begin{align*}
G_s(\zeta) = \gamma  \sum_{x'' \in \mathcal{X}} \bigg{[}   \max_{a'' \in \mathcal{A}} h_s^1(x'', a'') -  \max_{a'' \in \mathcal{A}} h_s^2(x'', a'') \bigg{]} p(x'' | \zeta ).
\end{align*}

Note that
\begin{align}
| G_s(\zeta) | \leq& \gamma  \sum_{x'' \in \mathcal{X}}   \max_{a'' \in \mathcal{A}} | \phi_s(x'', a'') |  p(x'' | \zeta ) \notag \\
%\leq& C \sum_{(x,a)\in \mathcal{X}\times\mathcal{A}} | \phi_s(\zeta) |,\label{Eq:Bound_G}
\leq& C \sum_{\zeta \in \mathcal{X}\times\mathcal{A}} | \phi_s(\zeta) |,\label{Eq:Bound_G}
\end{align}
where we have used the inequality $\displaystyle | \max_y f(y) - \max_y g(y) | \leq \max_y | f(y) - g(y) |$.

Then, $\phi_t$, at the point $\zeta$, i.e. $\phi_t(\zeta)$ satisfies the following equation
\begin{align*}
 \phi_t(\zeta) =&  \int_{0}^{t} \sum_{\zeta' \in \mathcal{X}\times\mathcal{A}}(G_{s}(\zeta')-\phi_{s}(\zeta'))A_{\zeta,\zeta'}\pi(\zeta')ds\nonumber\\
\phi_0(\zeta) =& 0,
\end{align*}

The latter, using (\ref{Eq:Bound_G}) and the boundedness of the elements $A_{\zeta,\zeta'}$, implies,
\begin{align}
|\phi_t(\zeta)|^{2}=& 2 \int_0^t \phi_s(\zeta) d \phi_s(\zeta) \notag \\
=&  2 \int_0^t \phi_s(\zeta)  \sum_{\zeta' \in \mathcal{X}\times\mathcal{A}}(G_{s}(\zeta')-\phi_{s}(\zeta'))A_{\zeta,\zeta'}\pi(\zeta') ds \notag \\
\leq& C \int_0^t \phi_s(\zeta) \sum_{\zeta\in \mathcal{X}\times\mathcal{A}} | \phi_s(\zeta) | ds.\notag
\end{align}

Then, summing over all possible $\zeta\in\mathcal{X}\times\mathcal{A}$ gives, due to the finiteness of the state space
\begin{align}
\sum_{\zeta\in \mathcal{X}\times\mathcal{A}}|\phi_t(\zeta)|^{2} \leq& C \int_0^t  \left|\sum_{\zeta\in \mathcal{X}\times\mathcal{A}} | \phi_s(\zeta) | \right|^{2} ds.\notag\\
\leq& C \int_0^t  \sum_{\zeta\in \mathcal{X}\times\mathcal{A}} | \phi_s(\zeta) |^{2}  ds.\notag
\end{align}

An application of Gronwall's inequality proves that $\phi_t(\zeta) = 0$ for all $0 \leq t \leq T$ and for all $\zeta\in\mathcal{X}\times\mathcal{A}$. Therefore, the solution $h_t$ is indeed unique.

\subsection{Proof of Convergence} \label{ProofOfConvergence}
We now combine the previous results of Sections \ref{RelativeCompactness} and \ref{Identification} to prove Theorem \ref{MainTheorem1}. Let $\rho^N$ be the probability measure corresponding to $(\mu^N, h^N)$. Each $\rho^N$ takes values in the set of probability measures $\mathcal{M} \big{(} D_E([0,T]) \big{)}$. Relative compactness, proven in Section \ref{RelativeCompactness}, implies that every subsequence $\rho^{N_k}$ has a further sub-sequence $\rho^{N_{k_m}}$ which weakly converges. Section \ref{Identification} proves that any limit point $\rho$ of $\rho^{N_{k_m}}$ will satisfy the evolution equation (\ref{EvolutionEquationIntroductionXavierInfiniteTime}). Equation (\ref{EvolutionEquationIntroductionXavierInfiniteTime}) has a unique solution (proven in Section \ref{Uniqueness}). Therefore, by Prokhorov's Theorem, $\rho^N$ weakly converges to $\rho$, where $\rho$ is the distribution of $(\mu, h)$, the unique solution of (\ref{EvolutionEquationIntroductionXavierInfiniteTime}). That is, $(\mu^N, h^N)$ converges in distribution to $(\mu, h)$.

\subsection{Analysis of the Limit Equation} \label{GlobalMinimum}

It is easy to show that there is a unique stationary point of the limit equation (\ref{EvolutionEquationIntroductionXavierInfiniteTime}) where $h = V$, the solution of the Bellman equation (\ref{HJBmain}). We define $\zeta$, $\zeta'$, and $A_{\zeta,\zeta'}$ as in Section \ref{Uniqueness}. Any stationary point $h$ of (\ref{EvolutionEquationIntroductionXavierInfiniteTime}) must satisfy
\begin{align}
0 =  \sum_{\zeta' \in \mathcal{X} \times \mathcal{A}} A_{\zeta,\zeta'}  \pi(\zeta') \bigg{(} r(\zeta') + \gamma \sum_{z \in \mathcal{X}} \max_{a''} h(z,a'') p(z | \zeta')  - h(\zeta') \bigg{)}.
\label{StationaryPoint}
\end{align}

Let $B$ be a matrix where $B_{\zeta,\zeta'} = A_{\zeta, \zeta'} \pi(\zeta')$. Since, by Lemma \ref{PositiveDefiniteLemma}, $A$ is positive definite and $\pi(\zeta') > 0$, $B$ is also positive definite. Therefore, we can re-write (\ref{StationaryPoint}) as
\begin{align}
0 = B \big{(} r + \gamma U - h \big{)},\nonumber
\end{align}
where $U(\zeta') = \displaystyle \sum_{z \in \mathcal{X}} \max_{a''} h(z,a'') p(z | \zeta')$. Since $B$ is positive definite, its inverse $B^{-1}$ exists and we can multiply both sides of this equation by $B^{-1}$ to yield
\begin{align}
0 = r + \gamma U - h.
\label{StationaryPoint2}
\end{align}
(\ref{StationaryPoint2}) is exactly the Bellman equation (\ref{HJBmain}), which has the unique solution $V$. Therefore, $h_t$ has a unique stationary point which equals the solution $V$ of the Bellman equation.

We now prove convergence of $h_t$ to $V$ for small $\gamma$. Define $\phi_t = h_t - V$ where $V$ is the unique solution to the Bellman equation (\ref{HJBmain}). We define now
\begin{align}
G_{\zeta,t} =  \sum_{ x'' \in \mathcal{X}} \bigg{[}   \max_{a'' \in \mathcal{A}} h_t(x'', a'') -  \max_{a'' \in \mathcal{A}} V(x'', a'') \bigg{]} p(x'' | \zeta).\nonumber
\end{align}

Note that
\begin{align}
| G_{\zeta,t} | \leq \sum_{x'' \in \mathcal{X}}   \max_{a'' \in \mathcal{A}} |  \phi_t(x'', a'') |  p(x'' | \zeta ).\nonumber
\end{align}

Then, $\phi_t(\zeta)$ satisfies
\begin{align}
d \phi_t =& - A  \bigg{(} \pi \odot ( \phi_t -  \gamma  G_t ) \bigg{)} dt,\nonumber
\end{align}
where $\odot$ is the element-wise product. The matrix $A$ is positive definite. Thus, $A^{-1}$ exists and is also positive definite.  Define the process
\begin{align}
Y_t =\frac{1}{2} \phi_t\cdot A^{-1} \phi_t.\nonumber
\end{align}

Then,
\begin{align}
d Y_t =&  \phi_t \cdot A^{-1}  d  \phi_t \notag \\
=&  - \phi_t\cdot  A^{-1}     A \bigg{(} \pi \odot ( \phi_t - \gamma  G_t ) \bigg{)} dt \notag \\
=& - \phi_t\cdot  \bigg{(} \pi \odot ( \phi_t - \gamma  G_t ) \bigg{)} dt  \notag \\
=&  - \pi \cdot \phi_t^2 dt + \gamma \phi_t\cdot( \pi \odot G_t ) dt,
\label{Yevolution}
\end{align}
where $\phi_t^2$ denotes the element-wise square $\phi_t^2 = \phi_t \odot \phi_t$.

Let us now study the second term in equation (\ref{Yevolution}). Let $\Gamma_t \vcentcolon =  \gamma \phi_t\cdot( \pi \odot G_t )$. Then,

\begin{align}
| \Gamma_t | \leq& \gamma  \sum_{\zeta \in \mathcal{X} \times \mathcal{A}} | \pi(\zeta) \phi_t(\zeta)  G_{\zeta,t} | \notag \\
\leq&   \frac{\gamma}{2} \sum_{\zeta \in \mathcal{X} \times \mathcal{A}} \pi(\zeta) \phi_t(\zeta)^2 +   \frac{\gamma}{2} \sum_{\zeta \in \mathcal{X} \times \mathcal{A}}  \pi(\zeta) G_{\zeta,t}^2 \notag \\
=&  \frac{\gamma}{2} \pi \cdot \phi_t^2 +   \frac{\gamma}{2} \sum_{\zeta \in \mathcal{X} \times \mathcal{A}}  \pi(\zeta) G_{\zeta,t}^2,\nonumber
%\gamma \sum_{x,a} \pi(\zeta)^2 \phi_t(\zeta)  \sum_{x''} \bigg{(}    \max_{a'' \in \mathcal{A}} h_s(x'', a'') -  \max_{a'' \in \mathcal{A}} v(x'', a'') \bigg{)} \pi(dx'' | x, a ).
\end{align}

We can bound the second term as
\begin{align}
  \sum_{\zeta \in \mathcal{X} \times \mathcal{A}}  \pi(\zeta) G_{\zeta,t}^2 \leq&    \sum_{\zeta \in \mathcal{X} \times \mathcal{A}}  \pi(\zeta) \sum_{x'' \in \mathcal{X}} \max_{a'' \in \mathcal{A}} |  \phi_t(x'', a'') |^2  p(x'' | \zeta ) \notag \\
  \leq&    \sum_{\zeta \in \mathcal{X} \times \mathcal{A}}  \pi(\zeta) \sum_{x'', a'' \in \mathcal{X} \times \mathcal{A}} |  \phi_t(x'', a'') |^2  p(x'' | \zeta ) \notag \\
   =&  K  \sum_{\zeta \in \mathcal{X} \times \mathcal{A}}  \pi(\zeta) \sum_{x'', a'' \in \mathcal{X} \times \mathcal{A}} |  \phi_t(x'', a'') |^2 \frac{1}{K} p(x'' | \zeta ) \notag \\
  =& K  \sum_{\zeta \in \mathcal{X} \times \mathcal{A}}  \pi(\zeta)  \phi_t(\zeta)^2  \notag \\
  =& K \pi \cdot \phi_t^2.\notag
  \end{align}

Consequently,
\begin{align}
| \Gamma_t | \leq \frac{\gamma}{2} \pi \cdot \phi_t^2 +   \frac{K \gamma}{2} \pi \cdot \phi_t^2.\nonumber
\end{align}

Suppose $\gamma < \frac{2}{1 + K}$. Then, there exists an $\epsilon > 0$ such that
\begin{align}
\frac{d Y_t}{d t} \leq - \epsilon \pi \cdot \phi_t^2.\nonumber
\end{align}
$Y_t$ is clearly decreasing in time $t$ and, since $A$ is positive definite, has a lower bound of zero. We also have the following upper bound using Young's inequality and the finite number of states in $\mathcal{X} \times \mathcal{A}$:
\begin{align}
Y_t =&  \sum_{\zeta, \zeta'\in \mathcal{X} \times \mathcal{A}} \phi_t(\zeta) A^{-1}_{\zeta,\zeta'}  \phi_t(\zeta') \notag \\
\leq& C \phi_t\cdot \phi_t,\nonumber
\end{align}
where $C > 0$. This leads to the lower bound $\phi_t^\cdot \phi_t \geq \frac{Y_t}{C}$ and the bound
\begin{align}
\frac{d Y_t}{d t} \leq& - \epsilon \min_{\zeta\in \mathcal{X} \times \mathcal{A}} \pi(\zeta)  \times  \phi_t\cdot \phi_t \notag \\
\leq& - C_0 Y_t,\nonumber
\end{align}
where $C_0 > 0$. By Gronwall's inequality,
\begin{align}
Y_t \leq Y_0 e^{- C_0 t}.\nonumber
\end{align}
Consequently,
\begin{align}
\lim_{t \rightarrow \infty} Y_t = 0,\nonumber
\end{align}
concluding the proof of Lemma \ref{L:ConvergenceInfiniteCase} due to the positive-definiteness of the matrix $A$.

\section{Proof of Convergence in Finite Time Horizon Case}\label{S:ProofFiniteCase}

In this section we address the proof of Theorem \ref{MainTheorem1}. The proof for the finite time horizon case is essentially exactly the same as the proof for the infinite time-horizon case. The main difference is that we can prove for any $0 < \gamma \leq 1$ that the limit equation $h_t$ converges to the Bellman equation solution $V$ as $t \rightarrow \infty$.

Let us begin by calculating the pre-limit evolution of the neural network output $Q^{N}_k(j, x,a)$. For convenience, let $\zeta = (x,a)$.
\begin{align}
Q_{k+1}^N(j, \zeta)  =& Q_{k}^N(j, \zeta)  + \frac{1}{\sqrt{N}} \sum_{i=1}^N C^{i,j}_{k+1} \sigma( W^i_{k+1} \cdot \zeta ) - \frac{1}{\sqrt{N}} \sum_{i=1}^N C^{i,j}_k \sigma( W^i_k \cdot \zeta ) \notag \\
=& Q_{k}^N(j, \zeta)  + \frac{1}{\sqrt{N}} \sum_{i=1}^N \bigg{(} C^{i,j}_{k+1} \sigma( W^i_{k+1} \cdot \zeta) -  C^{i,j}_k \sigma( W^i_k \cdot \zeta) \bigg{)} \notag \\
=& Q_{k}^N(j, \zeta)  + \frac{1}{\sqrt{N}} \sum_{i=1}^N \bigg{(} ( C^{i, j}_{k+1} - C^{i,j}_k )  \sigma( W^i_{k+1} \cdot \zeta)  + (  \sigma( W^i_{k+1} \cdot \zeta) -   \sigma( W^i_k \cdot \zeta)  ) C^{i,j}_k \bigg{)} \notag \\
=& Q_{k}^N(j,\zeta)  + \frac{1}{\sqrt{N}} \sum_{i=1}^N \bigg{(} ( C^{i,j}_{k+1} - C^{i,j}_k )  \bigg{[} \sigma(W^i_k \cdot \zeta) +  \sigma'( W^{i,\ast}_{k} \cdot \zeta_k ) \zeta\cdot ( W^i_{k+1} - W^i_{k} ) \bigg{]}   \notag \\
+& \bigg{[} \sigma'(W^i_k \cdot \zeta) \zeta \cdot ( W^i_{k+1} - W^i_k )   +  \frac{1}{2}\sigma''( W^{i, \ast \ast}_{k+1} \zeta) \left( ( W^{i}_{k+1} - W^{i}_{k} )\cdot \zeta \right)^2  \bigg{]} C^{i,j}_k \bigg{)},
\label{gEvolution100FiniteCase}
\end{align}
for points $W^{i,\ast}_{k}$ and $W^{i,\ast,\ast}_{k}$ in the line segment connecting the points $W^{i}_{k}$ and $W^{i}_{k+1}$. Let $\alpha^N = \frac{\alpha}{N}$. Substituting (\ref{SGDupdates00}) into (\ref{gEvolution100FiniteCase}) yields
\begin{align}
&Q_{k+1}^N(j, \zeta)  = Q_{k}^N(j, \zeta) + \frac{\alpha}{N^2}  \bigg{(} r_j + \gamma \max_{a' \in \mathcal{A}} Q^N_k(j+1, x_{j+1}, a') - Q^N_k(j, \zeta_j) \bigg{)} \sum_{i=1}^N \sigma(W^i_k \cdot \zeta_j ) \sigma(W^i_k \cdot \zeta)   \notag \\
&\quad+ \frac{\alpha}{N^2}  \sum_{m=0}^{J-1}  \bigg{(} r_m + \gamma \max_{a' \in \mathcal{A}} Q^N_k(m+1, x_{m+1}, a') - Q^N_k(m, \zeta_m) \bigg{)}  \sum_{i=1}^N \sigma'(W^i_k \cdot \zeta)    \sigma'(W^i_k \cdot \zeta_m ) \zeta_m\cdot \zeta  C^{i,m}_k C^{i,j}_k  \notag \\
&\quad + \mathcal{O}_{p}(N^{-3/2}).
\label{gEvolution1b00FiniteCase}
\end{align}

We can then re-write the evolution of $Q_k^N(j, x,a)$ in terms of the empirical measure $\nu_k^N$.
\begin{align}
&Q_{k+1}^N(j, \zeta)  = Q_{k}^N(j, \zeta) + \frac{\alpha}{N}  \bigg{(} r_j + \gamma \max_{a' \in \mathcal{A}} Q^N_k(j+1, x_{j+1}, a') - Q^N_k(j, \zeta_j) \bigg{)}  \la \sigma(w \cdot \zeta_j ) \sigma(w \cdot \zeta), \nu_k^N \ra   \notag \\
&\quad + \frac{\alpha}{N}  \sum_{m=0}^{J-1}  \bigg{(} r_m + \gamma \max_{a' \in \mathcal{A}} Q^N_k(m+1, x_{m+1}, a') - Q^N_k(m, \zeta_m) \bigg{)}  \la \sigma'(w \cdot \zeta)    \sigma'(w \cdot \zeta_m ) \zeta_m\cdot \zeta  c^{m} c^{j}, \nu_k^N \ra  \notag \\
&\quad+ \mathcal{O}_{p}(N^{-3/2}).
\label{gEvolution2FiniteCase}
\end{align}

(\ref{gEvolution2FiniteCase}) leads to an evolution equation for the re-scaled process $h_t$:
\begin{align}
&h_t^N(j,\zeta) = h_0^N(j,\zeta) \nonumber\\
&\quad +  \int_0^t  \sum_{\zeta'} \pi_j(\zeta') A_{\zeta,\zeta'}^0(s)   \bigg{(} r(j, \zeta') + \gamma \sum_{z} \max_{a'' \in \mathcal{A}} h_s^N(j+1,z, a'') p(z | \zeta') - h_s^N(j, \zeta') \bigg{)} ds \notag \\
&\quad+  \sum_{m=0}^{J-1}\int_0^t  \sum_{\zeta'} \pi_m(\zeta') A_{\zeta,\zeta'}^{j,m}(s)   \bigg{(} r(m, \zeta') + \gamma \sum_{z} \max_{a'' \in \mathcal{A}} h_s^N(m+1,z, a'') p(z | \zeta') - h_s^N(m, \zeta') \bigg{)} ds \notag \\
&\quad+ M_t^{N}(j,\zeta) +  \mathcal{O}_{p}(N^{-1/2}),\label{Eq:h_term_finiteTimeHorizon}
\end{align}
where the coefficients $A^0$ and $A^{j,m}$ are
\begin{align}
A^0_{\zeta,\zeta'}(s) =&  \alpha \la \sigma(w \cdot \zeta' ) \sigma(w \cdot \zeta), \mu_s^N \ra, \notag \\
A^{j,m}_{\zeta,\zeta'}(s) =& \alpha \la \sigma'(w \cdot \zeta')    \sigma'(w \cdot \zeta ) \zeta\cdot \zeta'  c^{j} c^{m}, \mu_s^N \ra.\notag
\end{align}

The fluctuation term $M_t^{N}(j, x,a)$ takes the form
\begin{align}
&M_t^{N}(j, \zeta) = \frac{\alpha}{N}  \sum_{k=0}^{\floor*{N t}-1} \bigg{[} \bigg{(} r_{k,j} + \gamma \max_{a' \in \mathcal{A}} Q^N_k(j+1, x_{k, j+1}, a') - Q^N_k(j, \zeta_{k,j}) \bigg{)}  \la \sigma(w \cdot \zeta_{k,j} ) \sigma(w \cdot \zeta), \nu_k^N \ra   \notag \\
&\quad+   \sum_{m=0}^{J-1}  \bigg{(} r_{k,m} + \gamma \max_{a' \in \mathcal{A}} Q^N_k(m+1, x_{k, m+1}, a') - Q^N_k(m, \zeta_{k,m}) \bigg{)}  \la \sigma'(w \cdot \zeta)    \sigma'(w \cdot \zeta_{k,m} ) \zeta_{k,m}\cdot \zeta  c^{m} c^{j}, \nu_k^N \ra \bigg{]} \notag \\
&\quad- \int_0^t  \sum_{\zeta\in \mathcal{X} \times \mathcal{A}} \pi_j(\zeta) A_{\zeta,\zeta'}^0(s)   \bigg{(} r(j, \zeta') + \gamma \sum_{z\in \mathcal{X} \times \mathcal{A}} \max_{a'' \in \mathcal{A}} h_s(j+1,z, a'') p(z | \zeta') - h_s(j,\zeta') \bigg{)} ds \notag \\
&\quad-  \int_0^t  \sum_{m=0}^{J-1}\sum_{\zeta'\in \mathcal{X} \times \mathcal{A}} \pi_m(\zeta') A_{\zeta,\zeta'}^{j,m}(s)   \bigg{(} r(m, \zeta') + \gamma \sum_{z\in \mathcal{X} \times \mathcal{A}} \max_{a'' \in \mathcal{A}} h_s(m+1,z, a'') p(z | \zeta') - h_s(m, \zeta') \bigg{)} ds. \notag
\end{align}

Using the same analysis as in the infinite time horizon case (see Lemma \ref{Mlemma}), we can show that $M_t^N \overset{p} \rightarrow 0$ as $N \rightarrow \infty$.

\subsection{Identification of the Limit, Relative Compactness, and Uniqueness} \label{ProofOfConvergencetoLimitEq}

Let $\rho^N$ be the probability measure of a convergent subsequence of $\left(\mu^N, h^N \right)_{0\leq t\leq T}$. Each $\rho^N$ takes values in the set of probability measures $\mathcal{M} \big{(} D_E([0,T]) \big{)}$. We can prove the following results.

\begin{lemma} \label{FiniteCaseLemma1}
The sequence $\rho^N$ is relatively compact in $\mathcal{M} \big{(} D_E([0,T]) \big{)}$.
\end{lemma}
\begin{proof}
The result is obtained by following the exact same steps as in the proofs of Lemmas \ref{L:CompactContainment}, \ref{MU:regularity}, \ref{H:regularity}, and \ref{L:RelativeCompactness}. Therefore, its proof will not be repeated here.
\end{proof}

\begin{lemma} \label{FiniteCaseLemma2}
Let $\rho^{N_k}$ be a convergent subsequence with a limit point $\rho$. Then, $\rho$ is a Dirac measure concentrated on $(\mu, h) \in D_E([0,T])$ and $(\mu, h)$ satisfies equation (\ref{EvolutionEquationIntroductionXavierFiniteTimeHorizon}).
\end{lemma}
\begin{proof}
The proof is exactly the same as in Lemma \ref{IdentificationInfiniteTime}, and we do not repeat it here. We only note here for completeness that due to the fact that the random variables $C_{0}^{i,j}, W^{i}_0$ are assumed to be mean zero, independent random variables (see Assumption \ref{A:Assumption2}), the terms  $A^{j,m}_{\zeta,\zeta'}(s)$ with $m\neq j$ will become zero in the limit as $N\rightarrow\infty$ in the expression for (\ref{Eq:h_term_finiteTimeHorizon}).
\end{proof}

\begin{lemma} \label{FiniteCaseLemma3}
The solution $(\mu, h)$ to the equation (\ref{EvolutionEquationIntroductionXavierFiniteTimeHorizon}) is unique.
\end{lemma}
\begin{proof}
The proof follows the same steps as in Section \ref{Uniqueness}, and we do not repeat it here.
\end{proof}

Combining Lemmas \ref{FiniteCaseLemma1}, \ref{FiniteCaseLemma2}, and \ref{FiniteCaseLemma3} proves that $(\mu^N, h^N) \overset{d} \rightarrow (\mu, h)$ as $N \rightarrow \infty$.

\subsection{Analysis of Limit Equation} \label{AnalysisLimitEqnFiniteTimeCase}

Let $\phi_t(j,x,a) = h_t(j,x,a) - V(j,x,a)$ where $V(j,x,a)$ is the solution to the Bellman equation (\ref{HJBmainFiniteT}). Note that $h_t(J,x, a) = V(J, x,a) = r(J,x)$ and thus $\phi_t(J,x,a) = 0$. Then,
\begin{align}
d \phi_t(J-1,x,a) =-  \sum_{(x', a')\in \mathcal{X}\times\mathcal{A}} \pi_{J-1}(x', a') A_{x,a,x', a'}  \phi_t(J-1, x', a') dt.
\end{align}

Let $\zeta = (x,a)$ and $\zeta' = (x', a')$. By Lemma \ref{PositiveDefiniteLemma} the matrix $A_{\zeta, \zeta'} $ is positive definite and recall that $\pi_{j-1}(\zeta')>0$ for every $\zeta'\in\mathcal{X}\times\mathcal{A}$. Therefore, using the same analysis as in Section \ref{GlobalMinimum}, $\displaystyle \lim_{t \rightarrow\infty} \phi_t(J-1, y) \rightarrow 0$.

In fact, using induction, we can prove that $\displaystyle \lim_{t \rightarrow \infty} \phi_t(j, x,a) = 0$ for $j = 0, 1, \ldots, J$. Indeed, let us assume that $\displaystyle \lim_{t \rightarrow \infty} \phi_t(j+1, y) = 0$ for each $\zeta \in \mathcal{X} \times \mathcal{A}$. Let $Y_t = \frac{1}{2} \phi_{t,j}\cdot A^{-1} \phi_{t,j}$ where $\phi_{t,j} = \phi_t(j, \cdot)$. The process $Y_t$ satisfies the differential equation

\begin{align}
d Y_t =&  \phi_{t,j}\cdot A^{-1}  d  \phi_{t,j} \notag \\
=&  - \phi_{t,j}\cdot  A^{-1}     A \bigg{(} \pi_j \odot ( \phi_{t,j} - \gamma  G_t ) \bigg{)} dt \notag \\
=& - \phi_{t,j}\cdot  \bigg{(} \pi_j \odot ( \phi_{t,j} - \gamma  G_t ) \bigg{)} dt  \notag \\
=&  - \pi_j \cdot \phi_{t,j}^2 dt + \gamma \phi_{t,j}^{\top}( \pi_j \odot G_{t, j+1} ) dt,
\end{align}
where the vector $G_{t, j+1}$ is given by
\begin{align}
G_{t, j+1}(\zeta) =  \sum_{ x'' \in \mathcal{X}} \bigg{[}   \max_{a'' \in \mathcal{A}} h_t(j+1, x'', a'') -  \max_{a'' \in \mathcal{A}} V(j+1, x'', a'') \bigg{]} p(x'' | \zeta ).
\end{align}

Let $\Gamma_t \vcentcolon =  \gamma \phi_{t,j}^{\top}( \pi_j \odot G_{t,j+1} )$. Then,
\begin{align}
| \Gamma_t | \leq& \gamma  \sum_{\zeta\in \mathcal{X}\times\mathcal{A}} | \pi_j(\zeta) \phi_{t,j}(\zeta)  G_{t, j+1}(\zeta) | \notag \\
\leq&   \frac{\gamma}{2} \sum_{\zeta\in \mathcal{X}\times\mathcal{A}} \pi_j(\zeta) \phi_{t,j}(\zeta)^2 +   \frac{\gamma}{2} \sum_{\zeta\in \mathcal{X}\times\mathcal{A}}  \pi_j(\zeta) G_{t, j+1}(\zeta)^2 \notag \\
=&  \frac{\gamma}{2} \pi_j \cdot \phi_{t,j}^2 +   \frac{\gamma}{2} \sum_{\zeta}  \pi_j(\zeta) G_{t, j+1}(\zeta)^2
%\gamma \sum_{x,a} \pi(x,a)^2 \phi_t(x,a)  \sum_{x''} \bigg{(}    \max_{a'' \in \mathcal{A}} h_s(x'', a'') -  \max_{a'' \in \mathcal{A}} v(x'', a'') \bigg{)} \pi(dx'' | x, a ).
\end{align}

We can bound the second term $\Gamma_t^2 \vcentcolon =  \frac{\gamma}{2} \sum_{\zeta}  \pi_j(\zeta) G_{t, j+1}(\zeta)^2$ as
\begin{align}
|\Gamma_t^2| =&  \frac{\gamma}{2} \sum_{\zeta\in \mathcal{X}\times\mathcal{A}}  \pi_j(\zeta) G_{t}(\zeta)^2 \notag \\
 \leq&  \frac{\gamma}{2}   \sum_{\zeta\in \mathcal{X}\times\mathcal{A}}  \pi_j(\zeta) \sum_{x''\in\mathcal{X}} \max_{a'' \in \mathcal{A}} |  \phi_t(j+1, x'', a'') |^2  p(x'' | \zeta ).
\end{align}
Consequently, $\displaystyle \lim_{t \rightarrow \infty} \Gamma_t^2 = 0$. We are now in a position to prove the convergence of $Y_t$. Similar to the analysis in Section \ref{GlobalMinimum}, we can show that, for $t > s$,
\begin{align}
Y_t \leq Y_s - C_0 \int_0^t Y_s ds + \int_s^t | \Gamma_s^2 | ds,
\end{align}
where $C_0 > 0$. We now choose an (arbitrary) $\epsilon > 0$. Since $\displaystyle \lim_{t \rightarrow \infty} \Gamma_t^2 = 0$, there exists a $T_0$ such that  $| \Gamma_t^2 | < \frac{\epsilon C_0}{3}$ for $t > T_0$. Suppose that there exists a $T_1 > T_0$ such that $Y_{t} > \epsilon$ for $t > T_1$. Then, for $t \geq T_1$,

\begin{align}
Y_t \leq& Y_{T_1} - C_0 \int_{T_1}^t Y_s ds + \int_{T_1}^t | \Gamma_s^2 | ds \notag \\
\leq& Y_{T_1} - C_0 \int_{T_1}^t \epsilon ds + \int_{T_1}^t \frac{\epsilon C_0}{3} ds \notag \\
\leq& Y_{T_1} - \frac{2 C_0 \epsilon }{3} (t - T_1).
\end{align}
This upper bound implies that $Y_t < 0$ for some $t > T_1$. However, $Y_t \geq 0$ for all $t \geq 0$ and thus this is a contradiction. Consequently, there exists a $T_2 > T_0$ such that $Y_{T_2} = \epsilon$.

Suppose that there exists a $T_4 > T_2$ such that $Y_{T_4} > \epsilon$. Define the time $T_3 = \max \{ t : T_2 \leq t \leq T_4, Y_t = \epsilon \}$. Then, we obtain
\begin{align}
Y_{T_4} \leq& Y_{T_3} - C_0 \int_{T_3}^{T_4} Y_s ds + \int_{T_3}^{T_4} | \Gamma_s^2 | ds \notag \\
\leq& Y_{T_3} - \frac{2 C_0 \epsilon }{3} (T_4 - T_3) \notag \\
\leq& \epsilon,
\end{align}
which is a contradiction.

Therefore, for any $\epsilon > 0$, there  exists a $T_2 > 0$ such that $Y_t \leq \epsilon$ for all $t \geq T_2$. Since $\epsilon$ is arbitrary, we have proven that

\begin{align}
\lim_{t \rightarrow \infty} Y_t = 0.
\end{align}

Therefore, if $\displaystyle \lim_{t \rightarrow \infty} \phi_{t, j+1} = 0$, we have shown that $\displaystyle \lim_{t \rightarrow \infty} \phi_{t, j} = 0$. By induction, $\displaystyle \lim_{t \rightarrow \infty} \phi_{t,j} = 0$ for $j = 0, 1, \ldots, J-1$. This concludes the convergence proof for Theorem \ref{MainTheorem1}.

\section{Proof that $A$ is positive definite-Lemma \ref{PositiveDefiniteLemma}}  \label{ProofOFPositiveDefinite}

We now prove Lemma \ref{PositiveDefiniteLemma}. Recall the matrix $A$ with elements $A_{\zeta,\zeta'}$ for $\zeta, \zeta' \in \{ \zeta^{(1)}, \ldots, \zeta^{(M)} \}$ where $\zeta^{(i)} \in \mathcal{S} \subset \mathbb{R}^d$. Furthermore, each $\zeta^{(i)}$ is distinct (i.e., $\zeta^{(i)} \neq \zeta^{(j)}$ for $i \neq j$). We prove that $A$ is positive definite under Assumption \ref{A:AssumptionPositiveDefinite} (or equivalently under Assumption \ref{A:AssumptionPositiveDefiniteFiniteTime}).

Let $U =  \bigg{(} U(\zeta^{(1)}), \ldots,  U(\zeta^{(M)}) \bigg{)}$, where $U(\zeta)$ is defined as

\begin{align}
U(\zeta) =& \sqrt{\alpha} \sigma(W \cdot \zeta ) +  \sqrt{\alpha}   C  \sigma'(W \cdot \zeta ) \zeta,
%=&  \sqrt{ \frac{\alpha}{M} } \sigma(W \cdot x ),
\end{align}
where $(W, C) \sim \mu_0$. Since $C$ is a mean-zero random variable which is independent of $W$,

\begin{align}
\mathbb{E} \bigg{[} U(\zeta) U(\zeta') \bigg{]} =  \mathbb{E} \bigg{[} \alpha  \sigma(W \cdot \zeta )  \sigma(W \cdot \zeta') + \alpha C^2 \sigma'(W \cdot \zeta )  \sigma'(W \cdot \zeta')  \zeta\cdot \zeta'\bigg{]} = A_{\zeta, \zeta'}.
\end{align}
Note that if $\sigma(\cdot)$ is an odd function (e.g., the tanh function) and the distribution of $W$ is even, then $A$ is a covariance matrix.

To prove that $A$ is positive definite, we need to show that $z\cdot A z > 0$ for every non-zero $z \in \mathbb{R}^M$.
\begin{align}
z\cdot A z =& z\cdot \mathbb{E} \bigg{[} U \cdot U \bigg{]}  z \notag \\
=& \mathbb{E} \bigg{[} ( z\cdot U )^2 \bigg{]} \notag \\
=& \alpha \mathbb{E} \bigg{[} \bigg{(}  \sum_{i=1}^M z_i  ( \sigma(\zeta^{(i)} \cdot W ) +  C  \sigma'(W \cdot \zeta^{(i)} ) \zeta^{(i)} )  \bigg{)}^2 \bigg{]}.
\end{align}

The functions $\sigma(\zeta^{(i)} \cdot W )$ are linearly independent since the $\zeta^{(i)}$ are in distinct directions (see Remark 3.1 of \cite{yIto}). Therefore, for each non-zero $z$, there exists a point $w^{\ast}$
such that
\begin{align}
 \sum_{i=1}^M z_i \sigma(\zeta^{(i)} \cdot w^{\ast} )  \neq 0.\nonumber
\end{align}

Consequently, there exists an $\epsilon > 0$ such that
\begin{align}
\bigg{(}  \sum_{i=1}^M z_i \sigma(\zeta^{(i)} \cdot w^{\ast} )  \bigg{)}^2 > \epsilon.\nonumber
\end{align}

Since $\sigma( w \cdot \zeta ) + c \sigma'(w \cdot \zeta) \zeta$ is a continuous function, there exists a set $B = \{ c, w : \norm{w - w^{\ast} } + \norm{c} < \eta \}$ for some $\eta > 0$ such that for $(c,w) \in B$
\begin{align}
\bigg{(}  \sum_{i=1}^M z_i  ( \sigma(\zeta^{(i)} \cdot W ) +  C  \sigma'(W \cdot \zeta^{(i)} ) \zeta^{(i)} )   \bigg{)}^2 > \frac{\epsilon}{2}.\nonumber
\end{align}

Then, we obtain that
\begin{align}
 \mathbb{E} \bigg{[} \bigg{(} \sum_{i=1}^M z_i  ( \sigma(\zeta^{(i)} \cdot W ) +  C  \sigma'(W \cdot \zeta^{(i)} ) \zeta^{(i)} )  \bigg{)}^2 \bigg{]} &\geq  \mathbb{E} \bigg{[} \bigg{(}  \sum_{i=1}^M z_i  ( \sigma(\zeta^{(i)} \cdot W ) +  C  \sigma'(W \cdot \zeta^{(i)} ) \zeta^{(i)} )  \bigg{)}^2 \mathbf{1}_{C, W \in B} \bigg{]} \notag \\
 &\geq \mathbb{E} \bigg{[} \frac{\epsilon}{2} \mathbf{1}_{C, W \in B} \bigg{]} \notag \\
 =& \frac{K \epsilon}{2},\nonumber
\end{align}
where $K > 0$. Therefore, for every non-zero $z \in \mathbb{R}^M$, we do have
\begin{align}
z\cdot A z > 0,\nonumber
\end{align}
and $A$ is positive definite, concluding the proof of the Lemma.
%\end{proof}

\section{Conclusion} \label{Conclusion}

We have proven that a single-layer neural network trained with the Q-learning algorithm converges in distribution to a random ordinary differential equation as the number of hidden units and the number of training steps become large. Our analysis of the limit differential equation proves that it has a unique stationary solution which is the solution of the Bellman equation, thus giving the optimal control for the problem. In addition, we provide conditions under which the limit differential equation converges to the stationary solution. As a by-product of our analysis, we obtain the limiting behavior of single-layer neural networks when trained on i.i.d. data with stochastic gradient descent under the widely-used Xavier initialization.

\textcolor{black}{In terms of future research directions, it would be interesting to eliminate the restriction of Lemma \ref{L:ConvergenceInfiniteCase} on  $\gamma<\frac{2}{1+K}$ in the infinite horizon setting. Instead, one would like this result to be true for all $\gamma\in(0,1)$. Another interesting direction is to extend the results of this paper to multi-layer neural networks.}

\end{document}